\documentclass{article}

\PassOptionsToPackage{numbers, compress}{natbib}

\usepackage[preprint]{neurips_2026}

\usepackage[utf8]{inputenc} 
\usepackage[T1]{fontenc}    
\usepackage{hyperref}       
\usepackage{url}            
\usepackage{booktabs}       
\usepackage{amsfonts}       
\usepackage{nicefrac}       
\usepackage{microtype}      
\usepackage{xcolor}         
\usepackage{graphicx}
\usepackage{subcaption}
\usepackage{amssymb}
\usepackage{algorithm}
\usepackage{algorithmic}
\usepackage{amsmath,amssymb,amsthm,amsfonts}
\usepackage{adjustbox}
\usepackage{multirow}
\usepackage{caption}
\usepackage{colortbl}
\usepackage{xcolor} 
\definecolor{lightgray}{gray}{0.9} 
\usepackage{graphicx}
\usepackage{wrapfig}
\usepackage{booktabs}
\usepackage{hyperref}
\usepackage{natbib}
\usepackage{microtype}
\usepackage{enumitem}
\usepackage{fancyhdr}
\usepackage{todonotes}
\usepackage[most]{tcolorbox}
\usepackage[toc,page,header]{appendix}
\usepackage{minitoc}

\renewcommand \thepart{}
\renewcommand \partname{}

\graphicspath{{figures/}}
\DeclareSymbolFont{matha}{OML}{txmi}{m}{it}
\DeclareMathSymbol{\varv}{\mathord}{matha}{118}

\newtheorem{theorem}{Theorem}

\theoremstyle{definition}
\newtheorem{assumption}{Assumption}
\newtheorem{remark}{Remark}
\newtheorem{counterexample}{Counterexample}

\definecolor{theoremblue}{RGB}{36, 92, 160}
\definecolor{theorembg}{RGB}{245, 248, 252}

\definecolor{thmgray}{RGB}{248,248,248}
\definecolor{thmline}{RGB}{120,120,120}

\newtcolorbox{neuraltheorem}[1][]{
	enhanced,
	breakable,
	colback=thmgray,
	colframe=thmline,
	boxrule=0.4pt,
	arc=1pt,
	left=6pt,
	right=6pt,
	top=6pt,
	bottom=6pt,
	before skip=8pt,
	after skip=8pt,
}

\newtcolorbox{neuralcounterexample}[1][]{
	enhanced,
	breakable,
	colback=white,
	colframe=black!35,
	boxrule=0pt,
	leftrule=0.7pt,
	arc=0pt,
	left=6pt,
	right=2pt,
	top=4pt,
	bottom=4pt,
	before skip=6pt,
	after skip=6pt,
	borderline west={0.7pt}{0pt}{black!35,dashed},
	#1
}

\title{Flatness and Gradient Alignment Are Both Necessary: Spectral-Aware Gradient-Aligned Exploration for Multi-Distribution Learning}

\author{%
  Aristotelis Ballas \\
  Department of Informatics and Telematics\\
  Harokopio University of Athens\\
  Omirou 9, Athens, Greece\\
  \texttt{aballas@hua.gr} \\
   \And
   Christos Diou \\
   Department of Informatics and Telematics \\
   Harokopio University of Athens\\
   Omirou 9, Athens, Greece \\
   \texttt{cdiou@hua.gr} \\
}

\begin{document}

\doparttoc 
\faketableofcontents


\maketitle

\begin{abstract}
  Sharpness-aware and gradient-alignment methods have been shown to improve generalization, however each family of methods targets a single geometric property of the loss landscape, while ignoring the other. In this paper, we show that this omission is structurally unavoidable and that both \textit{flatness} and \textit{gradient alignment} should be considered in multi-distribution learning settings. Specifically, we derive an excess-risk decomposition that yields two additive leading-order terms: (i) an alignment term, controlled by the trace of $\bar{H}^{-1}\Sigma_g$ and (ii) a curvature term, controlled by $\bar{H}$, where $\bar{H}$ is the average Hessian and $\Sigma_g$ is the covariance of the gradient across distributions. Notably, $\bar{H}$ appears inverted in one and non-inverted in the other. We further show, via a counterexample, that neither quantity bounds the other in general, so no algorithm targeting only one term can guarantee low excess risk. Motivated by this decomposition, we propose SAGE (Spectral-Aware Gradient-Aligned Exploration) that targets both terms. The curvature component replaces SAM's gradient-scaled perturbation with the polar factor of each layer's gradient matrix, computed via Newton–Schulz iteration, so that the ascent step probes all directions with similar magnitude. On the other hand, the alignment component injects isotropic noise at the descent step, the magnitude of which scales with cross-distribution gradient disagreement. Experiments on five domain-generalization and two multi-task learning benchmarks show that the proposed method establishes a new state-of-the-art on DomainBed and acts as a general-purpose improvement to base MTL solvers, remaining competitive with, or even surpassing, state-of-the-art methods\footnote{Code will be published upon acceptance.}.
\end{abstract}

\section{Introduction}
Understanding the geometry of loss landscapes~\cite{li2018visualizing, fort2019large} in overparameterized neural networks has 
become a central lens for interpreting optimization behavior and why certain solutions generalize better than others~\cite{wu2017towards, li2018visualizing, fort2019large, rangamani2020loss, keskar2016large, jiang2019fantastic}. Specifically, two geometric properties have emerged as particularly influential: the \textit{flatness}~\cite{foret2020sharpness, kwon2021asam, kim2022fisher, zhang2024domaininspired, ban2025samo} of the loss surface around a minimum, and the \textit{alignment}~\cite{mansilla2021domain, shi2021gradient, le2024gradient, ballas2025gradient, yu2020gradient, liu2021conflict, Senushkin_2023_CVPR} of gradients computed across different training distributions, e.g., domains in the case of Domain Generalization (DG)~\cite{wang2022generalizing, zhou_domain_2023}, or tasks, in Multi-Task Learning (MTL)~\cite{9392366}. 

Even though both flatness and gradient alignment have been studied extensively in isolation, existing methods typically pursue each property as a standalone objective~\cite{foret2020sharpness, yu2020gradient, mansilla2021domain, shi2021gradient, li2023enhancing, zhang2025preconditioned, ballas2025gradient}, while their interplay has received limited formal treatment~\cite{wang2023sharpness, kim2022fisher}. Intuitively, a flat minimum may sit in a region where distribution-specific gradients conflict, while an aligned minimum may lie in a sharp basin where small perturbations inflate the loss. We argue that neither property alone is sufficient and that robust 
generalization benefits from minima that are simultaneously flat across directions and exhibit improved gradient agreement across training 
sub-objectives.

To justify this claim, we derive an excess-risk decomposition
(Theorem~\ref{thm:decomposition}) that yields two additive leading-order terms:
(i) an alignment term scaling as $\mathrm{tr}(\bar{H}^{-1}\Sigma_g)$, where
$\bar{H}$ is the average curvature and $\Sigma_g$ is the
gradient covariance matrix, across distributions and (ii) a curvature term scaling as
$\mathrm{tr}(\bar{H})$. $\bar{H}$ appears \textit{inverted} in the first term and
\textit{non-inverted} in the other and, as a result, targeting only one of the two
during optimization can leave the other unconstrained. We formalize this
independence in Counterexample~\ref{thm:decoupling}, which shows that neither quantity
bounds the other, even in the simple quadratic family.

Motivated by this decomposition, and based on the fact that first-order
optimization methods (including first-order flatness-aware methods) ignore
curvature, we propose \textbf{SAGE} (Spectral-Aware Gradient-Aligned
Exploration) a method that jointly targets both curvature and gradient agreement
via principled heuristics. For the curvature term ($\mathrm{tr}(\bar{H})$), SAGE
replaces SAM's~\cite{foret2020sharpness} gradient-scaled ascent perturbation
with the polar factor $UV^T$ of each layer's gradient matrix, computed
efficiently via Newton--Schulz iteration~\cite{bernstein2024old}, scaled by the
layer Frobenius norm so that all spectral directions are probed with comparable
magnitude \footnote{We do not explicitly address the decomposition via
  second-order approximations~\cite{kim2022fisher, zhang2023gradient,
    shin2025sassha} to avoid additional computational burden. In addition, the
  same orthogonalization in the gradient matrix singular basis has been used in recent
  optimizers~\cite{jordan2024muon, gupta2018shampoo} and corresponds to the
  steepest-descent step under the spectral norm~\cite{bernstein2024old}.}
(details in Section \ref{sec:methods}). For the alignment term, SAGE injects
isotropic Gaussian noise at the descent step with magnitude scaling with
cross-distribution gradient disagreement, biasing the optimizer away from
high-conflict regions. Empirical evaluation on DG~\cite{wang2022generalizing,
  zhou_domain_2023} and MTL~\cite{9392366} benchmarks indicates that SAGE
compares favorably to both sharpness-aware and gradient-alignment baselines
across seven datasets.

Our key contributions can be summarized as follows:

\begin{itemize}
	\item We derive an excess-risk decomposition for multi-distribution learning (Theorem~\ref{thm:decomposition}) that splits the leading-order risk into two additive terms, an alignment term scaling as $\mathrm{tr}(\bar{H}^{-1}\Sigma_g)$ and a curvature term scaling as $\mathrm{tr}(\bar{H})$, with $\bar{H}$ appearing inverted in one and non-inverted in the other.
	
	\item We provide a separation result (Counterexample~\ref{thm:decoupling}) showing that the two terms are independently controllable for quadratic losses, showing that even in a simple scenario, targeting each property in isolation may prove insufficient.
	
	\item Based on the above, we propose SAGE, a method targeting both terms of the decomposition; a spectral perturbation (curvature) and a gradient alignment-driven noise injection (alignment).
	
	\item We evaluate our proposed method on five DG and two MTL benchmarks, demonstrating competitive results against both sharpness-aware and gradient-alignment baselines.
\end{itemize}

\section{Theoretical Analysis: Flatness and Gradient Alignment in Multi-Distribution Learning}
\label{sec:theory}

\subsection{Setup and Notation}
\label{sec:setup}

Let $\mathcal{X}$ be the input space and $\mathcal{Y}$ the label space. Let $e$ be a joint distribution $P^{e}_{XY}$ over $\mathcal{X}\times\mathcal{Y}$ (denoted as \emph{environment} in \cite{arjovsky2019irm}). For example, different distributions/environments may correspond to different \emph{domains} in the domain generalization literature or \emph{tasks} in the multi-task literature. Let $\mathcal{E}$ be a set of distributions and $\mathcal{P}$ a meta-distribution over $\mathcal{E}$. For each $e \in \mathcal{E}$, let $\mathcal{L}_e : \Theta \to \mathbb{R}$ denote a per-distribution loss, where $\Theta \subseteq \mathbb{R}^d$ is an open parameter space. The \emph{population risk} is:
\begin{equation}
	R(\theta) := \mathbb{E}_{e \sim \mathcal{P}}[\mathcal{L}_e(\theta)],
	\label{eq:population_risk}
\end{equation}
and the \emph{empirical multi-distribution risk} over $K$ iid training distributions $e_1, \dots, e_K \sim \mathcal{P}$ is
\begin{equation}
	\hat{R}(\theta) := \frac{1}{K}\sum_{k=1}^K \mathcal{L}_{e_k}(\theta).
	\label{eq:empirical_risk}
\end{equation}
This is the standard objective minimized by ERM. 

At any base point $\theta_0 \in \Theta$ we introduce the per-distribution gradient and Hessian,
\vspace{-0.1em}
\begin{equation}
	g_e := \nabla_\theta \mathcal{L}_e(\theta_0), \qquad H_e := \nabla_\theta^2 \mathcal{L}_e(\theta_0),
\end{equation}
and the following cross-distribution statistics,
\vspace{-0.1em}
\begin{equation}
	\bar{g} := \mathbb{E}_e[g_e], \qquad \bar{H} := \mathbb{E}_e[H_e], \qquad \Sigma_g := \mathrm{Cov}_e(g_e) = \mathbb{E}_e\!\left[(g_e - \bar{g})(g_e - \bar{g})^T\right].
	\label{eq:summary_stats}
\end{equation}
\vspace{-0.1em}
$\bar{H}$ is the average curvature targeted by sharpness-aware methods, $\Sigma_g$ is the cross-distribution gradient covariance targeted (implicitly or explicitly) by gradient-alignment methods, and $\bar{g}$  is the expected gradient direction, vanishing ($\bar{g} = 0$) at stationary points of $R$.

\begin{assumption}[Regularity]
	\label{ass:regularity}
	Each $\mathcal{L}_e$ is three times continuously differentiable (i.e., $C^3$) in $\theta$, with third derivatives uniformly bounded in a neighborhood of $\theta^\star := \arg\min_\theta R(\theta)$, and $\bar{H}(\theta^\star) \succ 0$ (Positive-Definite). The gradients $g_e(\theta^\star)$ have finite second moments under $\mathcal{P}$.
\end{assumption}

\subsection{Excess-Risk Decomposition}
\label{sec:decomposition}

The main result of our analysis is a decomposition of the expected excess risk of the first-order empirical minimizer into two additive terms, one of which depends on $\Sigma_g$ through $\bar{H}^{-1}$ and the other of which depends on $\bar{H}$ directly. 

\begin{neuraltheorem}
\begin{theorem}[Multi-distribution excess-risk decomposition]
	\label{thm:decomposition}
	Under Assumption~\ref{ass:regularity}, let $\hat{\theta}$ denote the minimizer of $\hat{R}$ over $K$ iid training distributions, and let $\xi \sim \mathcal{N}(0, \sigma^2 I)$ denote isotropic Gaussian noise of scale $\sigma > 0$. Then, for a  parameter $\theta = \hat\theta + \xi$:
	\begin{equation}
		\mathbb{E}\!\left[\mathbb{E}_{\xi}[R(\theta)]\right] - R(\theta^\star)
		\;=\; \underbrace{\frac{1}{2K}\,\mathrm{tr}\!\left(\bar{H}^{-1}\Sigma_g\right)}_{\text{\normalfont alignment term}}
		\;+\; \underbrace{\frac{\sigma^2}{2}\,\mathrm{tr}(\bar{H})}_{\text{\normalfont curvature term}}
		\;+\; O(K^{-3/2}) \;+\; O(\sigma^3),
		\label{eq:decomposition}
	\end{equation}
	where the outer expectation is over the iid sampling of training distributions, $\bar{H} := \nabla^2 R(\theta^\star)$, and $\Sigma_g := \mathrm{Cov}_e(\nabla\mathcal{L}_e(\theta^\star))$.
\end{theorem}
\end{neuraltheorem}
\paragraph{Proof sketch.} To prove Theorem~\ref{thm:decomposition}, we expand 
$R$ quadratically around $\theta^\star$, where the linear term vanishes since 
$\bar{g}(\theta^\star)=0$. We then express $\hat\theta - \theta^\star = -\bar 
H^{-1}\hat g + O_P(K^{-1})$ via the implicit-function argument standard in 
M-estimator asymptotics~\cite{van2000asymptotic}. Substituting into the 
quadratic and taking expectation over the iid distribution sampling, with 
$\mathrm{Cov}(\hat g) = \tfrac{1}{K}\Sigma_g$, yields the alignment term. 
Separately, expanding $R$ around $\hat\theta$ under random perturbations $\xi 
\sim \mathcal{N}(0, \sigma^2 I)$, e.g. due to optimization noise, produces the 
curvature term. Full proof in Appendix~\ref{app:thm1_proof}.

\begin{remark}
	\label{rem:why_both}
	The two terms exhibit a structural trade-off, i.e., $\bar{H}$ appears 
	inverted in the alignment term and non-inverted in the curvature term. As a 
	result, minimizing one term by reshaping the Hessian strictly inflates the 
	other. Considered in isolation, the alignment term implies that excess risk 
	is minimized by maximizing $\bar{H}$. The curvature term mitigates 
	this by explicitly penalizing excessive sharpness, ensuring a balanced 
	optimization objective.
\end{remark}

\subsection{Decoupling of Flatness and Alignment}
\label{sec:decoupling}

Theorem~\ref{thm:decomposition} gives an additive decomposition into a flatness-controlled term and an alignment-controlled term. The natural next question is whether the two terms are linked, i.e. an inequality of the form ``small $\mathrm{tr}(\bar H)$ implies small $\mathrm{tr}(\bar H^{-1}\Sigma_g)$'' that would allow a single criterion to bound both. The following counterexample shows that such link may not exist, even in simple settings.

\begin{counterexample}[Decoupling of flatness and gradient alignment]
	\label{thm:decoupling}
	Consider the family of multi-distribution learning problems with per-distribution losses
	\begin{equation}
		\mathcal{L}_e(\theta) = \tfrac{1}{2}\theta^T A\theta + b_e^T \theta, \qquad \mathbb{E}_e[b_e] = 0,
		\label{eq:quadratic_family}
	\end{equation}
	parameterized by a symmetric PD matrix $A \in \mathbb{R}^{d\times d}$ and a collection of linear coefficients $\{b_e\}$ satisfying $\mathbb{E}_e[b_e] = 0$. For every $M > 0$, there exist instances of Eq.~\eqref{eq:quadratic_family} such that:
	\begin{enumerate}[label=(\roman*)]
		\item \textbf{Flat but misaligned:} $\mathrm{tr}(\bar H) \le M^{-1}$ and $\mathrm{tr}(\bar H^{-1}\Sigma_g) \ge M$.
		\item \textbf{Aligned but sharp:} $\mathrm{tr}(\bar H^{-1}\Sigma_g) \le M^{-1}$ and $\mathrm{tr}(\bar H) \ge M$.
	\end{enumerate}
	In particular, neither $\mathrm{tr}(\bar H)$ nor $\mathrm{tr}(\bar H^{-1}\Sigma_g)$ can be bounded above by any function of the other alone.
\end{counterexample}

\paragraph{Proof sketch.} For the quadratic family, $\bar H = A$ depends only on $A$ and $\Sigma_g = \mathbb{E}_e[b_e b_e^T]$ depends only on $\{b_e\}$, so the two quantities can be specified independently. Setting $A = \lambda I$ and choosing $\{b_e\}$ supported on a single direction with prescribed variance yields both constructions, by an appropriate choice of $\lambda$. Full proof details in Appendix~\ref{app:thm2_proof}.

\begin{remark}
  \label{rem:decoupling}
  This decoupling result in the quadratic setting indicates that an algorithm targeting only one of the two quantities may increase the generalization error resulting from the other. Thus, both flatness and alignment terms need to be considered during optimization. 
\end{remark}

To make Theorem~\ref{thm:decomposition} and Counterexample~\ref{thm:decoupling} concrete, we construct in Appendix~\ref{app:motivating_example} a two-dimensional learning problem in which the support of $\Sigma_g$ coincides with the flat eigendirection of $\bar H$. The example realizes and exhibits two distinct failure modes along orthogonal eigendirections.

\section{Spectral-Aware Gradient-Aligned Exploration}
\label{sec:methods}

In this section, we describe SAGE, a single algorithm that targets the 
excess-risk decomposition terms of Theorem~\ref{thm:decomposition} via two 
separate principled heuristic mechanisms. The two components introduced below 
are tractable surrogates designed to probe each term and are introduced to avoid
2nd-order computational overhead.

\begin{figure}[htbp]
	\includegraphics[width=1\linewidth]{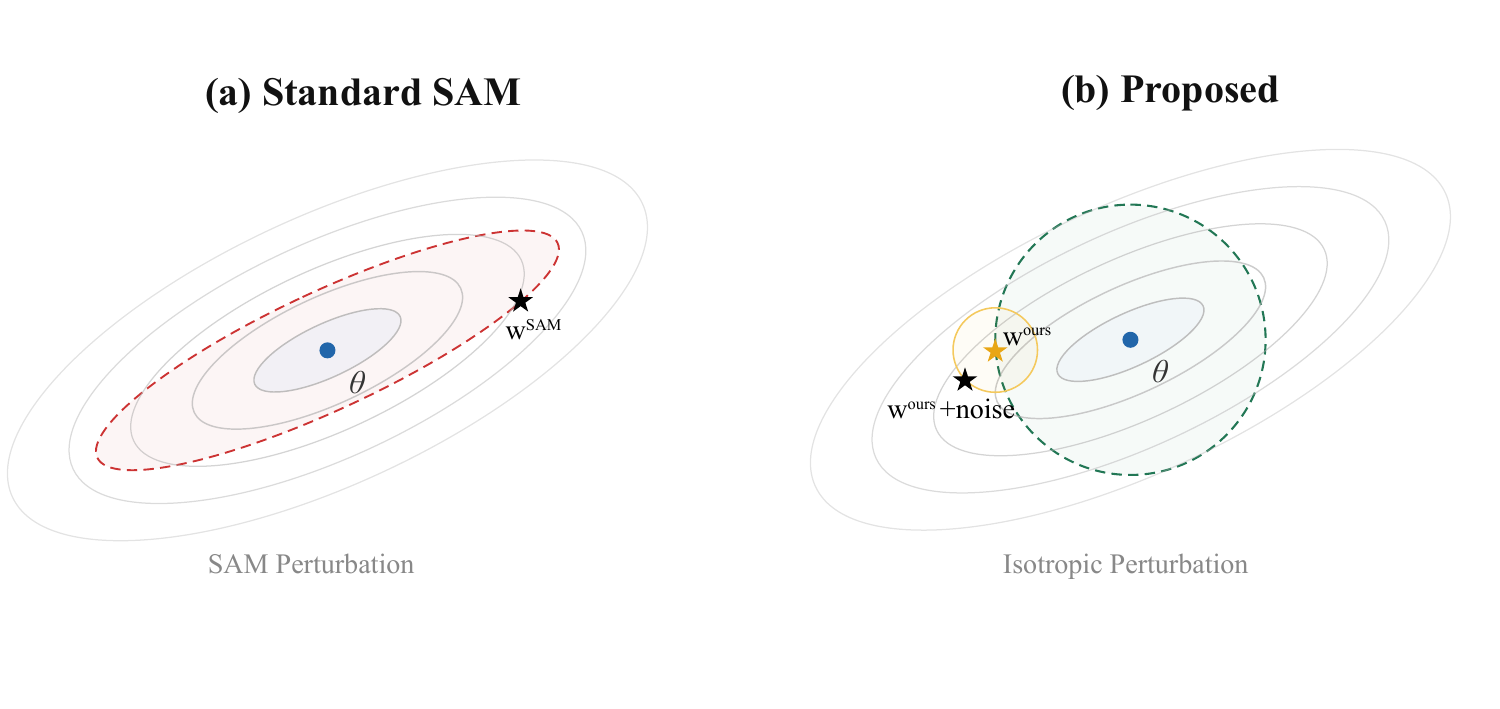}
	\vspace{-3em}
	\captionof{figure}{\textbf{Comparison of standard SAM and our proposed 
			method.} (a)~Standard SAM computes the perturbation as 
		$\epsilon_{\text{SAM}} = \rho\,\nabla_\theta \mathcal{L} / 
		\|\nabla_\theta \mathcal{L}\|_2$, determined solely by the gradient 
		direction (dashed ellipse), and the descent step lands at 
		$w^{\text{SAM}}$~($\bigstar$). 
		(b)~We replace the gradient with its orthogonal polar factor 
		($G \rightarrow UV^T$), setting all singular values to one and scaling by 
		the layer norm, yielding a perturbation 
		$\epsilon' = \rho \|W\|_F \cdot UV^T$ that probes all directions with 
		equal magnitude (dashed circle). After the descent step lands at 
		$w^{\text{ours}}$~(\textcolor{orange}{$\bigstar$}), isotropic noise 
		scaled by cross-distribution gradient disagreement perturbs the update 
		(orange circle), discouraging convergence to regions of gradient conflict.}
		\label{fig:fig1}
	\end{figure}

\subsection{Targeting the Curvature Term}
\label{sec:sage_spectral}

Motivated by the curvature term in Theorem~\ref{thm:decomposition} we follow 
the logic of SAM~\cite{foret2020sharpness} and SAM-like~\cite{zhuang2022surrogate, wang2023sharpness, kwon2021asam, 
li2023enhancing, zhang2024domaininspired, ban2025samo} methods, which have empirically proven to be effective in multi-distribution settings, but propose an 
alternative for the ascent perturbation step. Specifically, SAM-like methods derive the perturbation by approximating the inner maximization 
$\max_{\|\epsilon\|\le\rho}\mathcal{L}(\theta+\epsilon)$ with a first-order 
Taylor expansion, $\mathcal{L}(\theta+\epsilon) \approx \mathcal{L}(\theta) + 
\epsilon^T\nabla\mathcal{L}(\theta)$, whose maximizer over the $L_2$ ball 
admits the closed form $\epsilon_{\text{SAM}} = 
\rho\,\nabla\mathcal{L}/\|\nabla\mathcal{L}\|_2$. As apparent, the above 
perturbation direction ignores curvature information, is determined by the 
gradient vector alone, and is therefore probed by dominating gradients during training\footnote{Even though iterated SAM dynamics have been shown to implicitly minimize $\lambda_{\max}(\bar H)$~\cite{wen2023doessharpnessawareminimizationminimize}, it is a different functional of the Hessian from the trace $\mathrm{tr}(\bar H)$ identified by Theorem~\ref{thm:decomposition} as the relevant target.}. To address the above limitation, we propose replacing SAM's perturbation with an isotropic perturbation during the ascent step, inspired by preconditioning methods such as Shampoo~\cite{gupta2018shampoo} and Muon~\cite{jordan2024muon}. These methods have empirically shown to achieve effective preconditioning by orthogonalizing each $2D$ weight (i.e., parameter) matrix (e.g., layer) of a network independently. We follow the same heuristic to expand the exploration of additional perturbation directions of SAM. 

\paragraph{Spectral perturbation via orthogonalized gradients.}
For each weight matrix $W$ with gradient $G = \nabla_W \mathcal{L}$, we 
compute the orthogonal polar factor $Q = UV^T$, where $G = U\Sigma V^T$ is the 
SVD of $G$. This amounts to setting every singular value of the gradient to 
one, so that $Q = \sum_{i=1}^{r} u_i v_i^T$ retains the directional structure 
of $G$, i.e.\ the left and right singular vectors, while equalizing the 
magnitude across all singular directions. A perturbation along $Q$ therefore 
probes the loss surface with equal magnitude in every spectral direction of 
the gradient. Since computing the exact polar factor via SVD adds a computational burden of $\mathcal{O}(\min\{m,n\} \cdot mn)$ per weight matrix, we instead approximate $Q$ with the Newton--Schulz iteration~\cite{bernstein2024old, guo2006schur, bjorck1971iterative, kovarik1970some}, which converges cubically to the orthogonal factor using only matrix multiplications:
\begin{equation}
	X_0 = \frac{G}{\|G\|_F}, \qquad
	X_{k+1} = \tfrac{1}{2}\, X_k \bigl(3I - X_k^T X_k\bigr).
\end{equation}
After $T$ iterations ($T{=}5$ suffices in practice), $X_T \approx UV^T$. 
Because each step involves only two matrix products of the same shape as $G$, 
the overhead on modern GPU hardware is negligible. For bias vectors and other 
one-dimensional parameter groups, where matrix polar decomposition is not 
applicable, we fall back to the standard $L_2$-normalized perturbation 
$\epsilon = \rho\, g / \|g\|_2$, consistent with vanilla SAM.

\paragraph{Scale-adaptive radius.} Furthermore, to address the limitation from the well-documented interaction between SAM's fixed 
perturbation radius and the scale invariance of modern 
architectures~\cite{dinh2017sharp}, we scale the orthogonalized radius by the Frobenius norm of each weight matrix $\|W\|_F$, following the intuition of adaptive sharpness~\cite{kwon2021asam}. The final resulting spectral perturbation for weight matrix $W$ is therefore $\epsilon_W^{\text{spec}} = \rho \, \|W\|_F \, U V^T$,where $U, V$ are from the SVD of $G = \nabla_W \mathcal{L}$. Under this 
formulation, every singular direction receives a perturbation of identical 
magnitude $\rho \|W\|_F$, and the perturbation rescales with the current 
parameter norm, ensuring that the sharpness measure remains invariant to 
reparameterizations of the form $W \mapsto \alpha W$. More details in
Appendix~\ref{app:scale_invariance}.

\subsection{Targeting the Alignment term}
\label{sec:grad_alignment}

\begin{figure}[htbp]
	\centering
	\includegraphics[width=\textwidth]{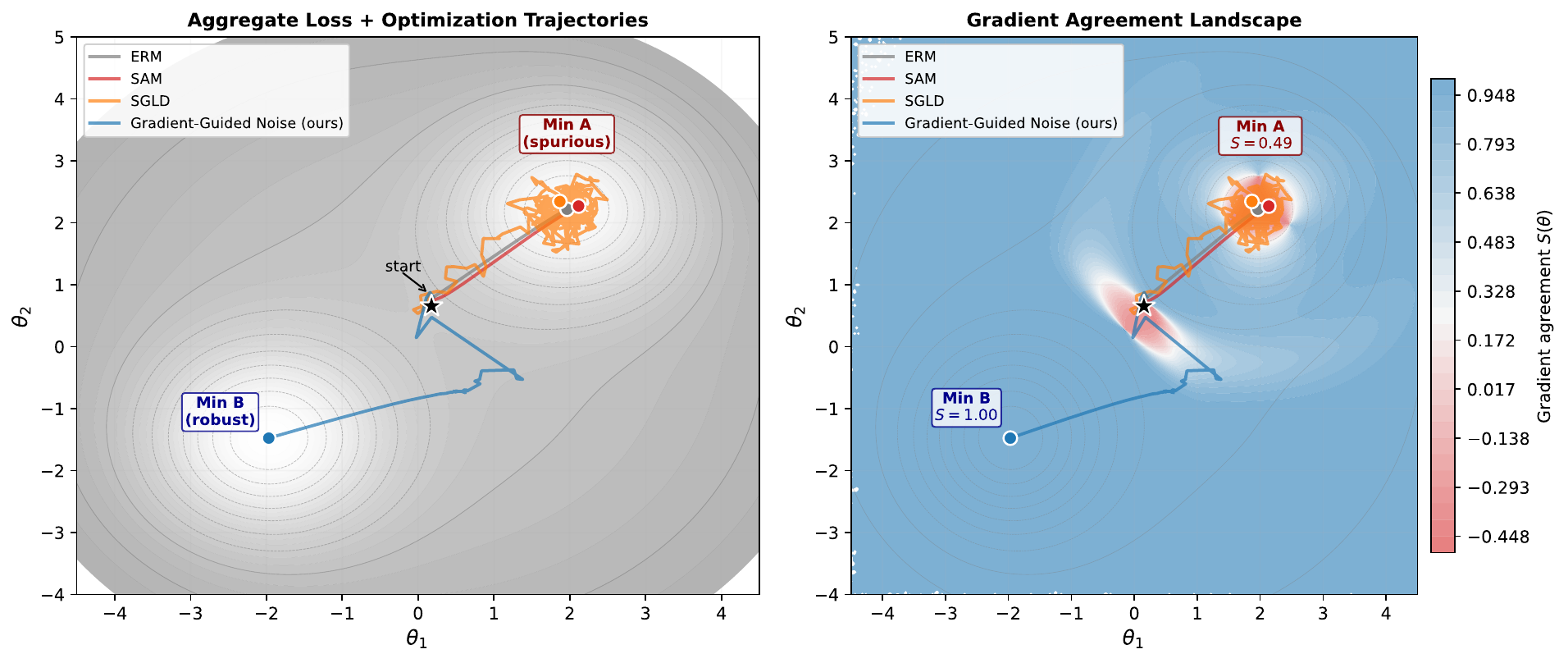}
	\caption{\textit{Left:} Aggregate loss contours with optimization trajectories from a shared starting point~($\star$). ERM (grey), SAM (red), 
		and SGLD (orange) converge to Minimum~A, whereas only gradient-aligned noise (blue, ours) escapes to Minimum~B. \textit{Right:} The same trajectories overlaid on the gradient agreement landscape $S(\theta)$, where \textcolor{cyan}{blue} and \textcolor{red}{red} indicate \textcolor{cyan}{high} and \textcolor{red}{low} cross-domain gradient agreement, respectively.}
	\label{fig:grad_alignment_toy}
\end{figure}

Motivated by the alignment term in Theorem~\ref{thm:decomposition}, we 
introduce a gradient-agreement-based regularizer at the descent step. Since the 
cross-distribution covariance $\Sigma_g$ is intractable to estimate and 
differentiate through at the scale of modern networks, we replace it with a 
tractable scalar proxy of cross-distribution gradient \emph{disagreement} and 
use it to modulate noise that biases the optimizer away from high-conflict 
regions, rather than minimizing $\mathrm{tr}(\bar H^{-1}\Sigma_g)$ directly.

In several settings, such as domain generalization~\cite{wang2022generalizing, zhou_domain_2023, 10233054}, multi-task learning~\cite{9392366}, the training signal decomposes into $K$ distinct distributions, each  
corresponding to a different domain, data source, or task. The overarching goal 
is to learn a model that performs well on all sub-tasks and unseen OOD data. 
Intuitively, \textit{directions of parameter update that are beneficial across 
all distributions are more likely to generalize to new ones}. This intuition 
has been formalized under gradient alignment~\cite{yu2020gradient, mansilla2021domain, shi2021gradient, le2024gradient}, where it has been shown that under certain assumptions~\cite{ballas2025gradient} it is a necessary 
condition for domain-invariance. Specifically, when gradients of different 
tasks or domains point in similar directions, the shared gradient direction 
captures features that are invariant across distributions. On the 
contrary, when they conflict the model is being pulled toward domain- or 
task-specific shortcuts.

Formally, as defined in Section~\ref{sec:setup}, given distribution loss functions $\mathcal{L}_{e_k}(\theta)$ for $K$ \textit{source} distributions $e_1, \dots, e_K \sim \mathcal{P}$, where in our case a distribution or environment $e_k$ may represent a domain, a task, or any other meaningful partition of the training data, we measure the agreement among gradients $g_{e_k} = \nabla_\theta \mathcal{L}_{e_k}(\theta)$ using the average pairwise cosine similarity, $S(\theta) = \frac{2}{K(K-1)} \sum_{i < j} \frac{\langle g_{e_i}, g_{e_j} \rangle}{\|g_{e_i}\| \, \|g_{e_j}\|}$.

When $S(\theta) \approx 1$, all distributions agree on the update direction, 
whereas when $S(\theta) \leq 0$, gradients conflict, indicating that the 
optimizer is navigating a region where domain- or objective-specific 
features dominate. Rather than modifying the gradient direction directly (as in 
PCGrad~\cite{yu2020gradient}) and to avoid further computational burden via 
second-order approximation techniques, we adopt a noise-injection strategy 
inspired by the recently proposed Gradient-Guided Annealing (GGA) 
framework~\cite{ballas2025gradient}. We inject isotropic Gaussian noise into 
the gradient at the descent step, with magnitude inversely proportional to the 
gradient agreement $\beta = \gamma \, (1 - S(\theta))$,
where $\gamma > 0$ is a hyperparameter controlling the maximum noise scale. The 
updated gradient then becomes $g' = \nabla_\theta \mathcal{L}(\theta + \epsilon) + \beta \, \xi, \quad \xi \sim \mathcal{N}(0, I)$.

The intuition behind this design is twofold. When gradients are well-aligned 
($S \approx 1$) the additive noise vanishes, whereas when gradients conflict 
($S \ll 1$) the noise discourages update directions that benefit some 
distributions at the expense of others. We illustrate the behavior between ERM, 
SAM~\cite{foret2020sharpness}, SGLD\footnote{SGLD or Stochastic Gradient 
Langevin Dynamics adds constant isotropic noise during 
optimization.}~\cite{welling2011bayesian, teh2016consistency}, and the proposed 
noise mechanism in Fig.~\ref{fig:grad_alignment_toy}.

\subsection{Spectral-Aware Gradient-Aligned Exploration}
\label{sec:sage}

We now describe the complete SAGE algorithm, which unifies the spectral 
perturbation and gradient-alignment into a single training procedure.
At each iteration, SAGE: (i) computes per-distribution gradients $g_{e_k}$ and 
aggregate gradient from a single forward pass, (ii) evaluates the gradient 
agreement $S(\theta)$ and sets the noise scale $\beta = \gamma(1 - S(\theta))$, 
(iii) constructs the spectral perturbation by orthogonalizing each weight 
matrix's gradient via the Newton-Schulz iteration and scaling by $\rho\|W\|_F$, 
(iv) performs a forward-backward pass at the perturbed point $\theta + 
\epsilon$, and (v) restores the clean weights, adds isotropic gradient 
alignment-driven Gaussian noise $\beta\ \mathcal{N}(0, I)$ to the 
adversarial gradient, and finally applies the base optimizer update step. 
A detailed summary in pseudocode of the SAGE algorithm is provided in Appendix~\ref{app:sage_algorithm}, whereas a conceptual illustration is provided in Fig. \ref{fig:fig1}.

\section{Related Work}
\label{sec:related_work}

In this section we situate SAGE within the literature of DG and MTL, and provide an overview of the most related works. An extended discussion regarding the landscape of the two fields is provided in Appendix~\ref{app:related_work}.

\paragraph{Sharpness-aware minimization.}
The observation that flat minima generalize better than sharp
ones~\cite{keskar2016large} motivated Sharpness-Aware Minimization
(SAM)~\cite{foret2020sharpness}, which performs a gradient ascent step followed
by a descent step so that the optimizer preferentially converges to flat regions
of the loss landscape. Several refinements have since been proposed, such as
GSAM~\cite{zhuang2022surrogate} which introduces a surrogate gap objective,
while ASAM~\cite{kwon2021asam} adapts perturbation magnitudes to the parameter 
scale, and SAGM~\cite{wang2023sharpness} additionally encourages the clean-loss 
and perturbed-loss gradients to remain aligned, jointly minimizing sharpness and
the surrogate gap. In the DG setting, DISAM~\cite{zhang2024domaininspired}
showed that SAM's perturbation is biased toward the domain with the largest
gradient and proposed a domain-loss variance constraint to counteract this
effect, while SWAD~\cite{cha2021swad} achieves flatness passively by averaging
model weights along the training trajectory. In MTL,
SAMO~\cite{ban2025samo} combines global and per-task perturbation information
through a zeroth-order estimator. An important note is the majority of all 
sharpness-aware algorithms incorporate the perturbation method of the initially proposed SAM~\cite{foret2020sharpness}.

\paragraph{Gradient-alignment methods.}
A complementary line of work uses agreement among per-environment or per-task
gradients as an indicator of domain invariance. In DG,
Fish~\cite{shi2021gradient} approximately maximizes the inner product of
per-domain gradients, while Mansilla et al.~\cite{mansilla2021domain} adapt 
gradient surgery~\cite{yu2020gradient} to mute gradient dimensions where 
domains disagree in sign. More recently, GGA~\cite{ballas2025gradient} searches 
for parameter configurations with well-aligned domain gradients via a 
simulated-annealing-inspired procedure. In MTL, PCGrad~\cite{yu2020gradient} 
projects conflicting task gradients, CAGrad~\cite{liu2021conflict} maximizes 
worst-case per-task improvement, and Nash-MTL~\cite{navon2022multi} frames the 
problem as a bargaining game. FAMO~\cite{liu2023famo} circumvents the 
$\mathcal{O}(K)$ per-iteration cost of gradient manipulation by working in 
log-loss space with amortized weight updates. However, these methods enforce gradient alignment without considering the curvature of the resulting solution.

\paragraph{Positioning of SAGE.}
The most closely related prior work is SAGM~\cite{wang2023sharpness}, which
also augments SAM with a gradient-matching term. However, SAGM's matching
aligns the clean-loss and perturbed-loss gradients, both computed on the
\emph{aggregate} data, and does not involve per-environment gradient
information. SAGE differs in two respects: (i) its spectral perturbation
replaces SAM's gradient-scaled ascent with the polar factor of each layer's
gradient, computed via Newton--Schulz iteration and scaled by the layer's
Frobenius norm, yielding a reparameterization-invariant sharpness probe; and
(ii) its noise injection at the descent step is scaled by the degree of
cross-environment gradient conflict, targeting the alignment term
$\mathrm{tr}(\bar H^{-1}\Sigma_g)$ identified in
Theorem~\ref{thm:decomposition}. Unlike methods that modify the gradient
direction (PCGrad, Fish) or solve auxiliary optimization problems (Nash-MTL,
FAMO), SAGE's noise injection is lightweight and requires no per-environment
gradient storage beyond what is needed for computing pairwise cosine
similarities. SAGE can also operate as a drop-in replacement for the optimizer's ascent and descent steps, making it compatible with existing gradient manipulation methods in MTL (as demonstrated in Table~\ref{tab:combined}) and with the standard DomainBed protocol in DG (Table~\ref{table:baseline-results}).

\section{Experiments}
\label{sec:experiments}

\subsection{Experimental setting}
\label{sec:experimental_settings}

\paragraph{Datasets and Protocol.} Our method is evaluated in two different 
experimental settings, namely Domain Generalization (DG)~\cite{zhou_domain_2023} and Multi-Task learning~\cite{9392366}. For DG, we 
follow the protocol of the widely adopted and challenging DomainBed~\cite{gulrajani2020domainbed} benchmark, and conduct experiments on 
five image classification datasets. Specifically, we follow the 
leave-one-domain-out evaluation protocol for PACS~\cite{li2017deeper}, 
VLCS~\cite{fang2013unbiased}, OfficeHome~\cite{venkateswara2017deep}, 
TerraIncognita~\cite{beery2018recognition}, and 
DomainNet~\cite{peng2019moment}, and report the average top-1 accuracy over 3 
runs based on training-domain split validation. For the MTL setting, we follow 
the standard protocol in recent literature~\cite{ban2025samo, liu2023famo, 
navon2022multi} and evaluate SAGE on Cityscapes~\cite{cordts2016cityscapes} (semantic segmentation and depth estimation), and NYU-v2~\cite{silberman2012indoor} (semantic segmentation, depth estimation, and surface normal prediction).

\paragraph{Implementation Details.} For the DG experiments, we follow recent literature~\cite{liexploring, zhang2024domaininspired} and finetune a 
CLIP~\cite{radford2021learning} pretrained ViT-B/16~\cite{dosovitskiy2021an} 
model. For the MTL setting, we follow~\cite{10.5555/3692070.3692179, 
liu2021conflict, liu2023famo, navon2022multi} and employ MTAN~\cite{liu2019end} 
as the shared backbone, with task-specific attention modules built on top
of SegNet~\cite{badrinarayanan2017segnet}. Dataset and training hyperparameter details are provided in Appendix~\ref{app:experiments}.

\subsection{Main results}
\label{sec:experimental_results}

Table \ref{table:baseline-results} presents the results on the DomainBed benchmark. SAGE achieves an average accuracy of 78.9\%, establishing a new state-of-the-art in this particular setting. Notably, SAGE even surpasses 
recent sharpness-aware adaptations tailored for DG, including SAGM (76.4\%), 
DISAM (76.5\%), and the BOA regularizer (78.1\%), which is applied to SAM-pretrained backbones. Extending beyond domain shift, we demonstrate 
SAGE's versatility on the NYU-v2 and Cityscapes MTL benchmarks (Table 
\ref{tab:combined}). The results illustrate that integrating SAGE
improves the performance of underlying base MTL solvers. For instance, augmenting FairGrad~\cite{10.5555/3692070.3692179} with SAGE (SAGE-FairGrad) 
yields improvements, notably reducing the per-task performance drop against the single-task (STL) ($\Delta m\%\downarrow$) from $-4.96$ to $-5.63$ on NYU-v2, and from $3.90$ to $2.63$ on Cityscapes. Similar relative gains are observed when SAGE is applied to Linear Scalarization (LS) and MGDA~\cite{sener2018multi}. 

\begin{table*}[h]
	\centering
	\small
	\caption{{\bf Comparison of \textbf{SAGE} on DomainBed}. The top out-of-domain accuracies on five domain generalization benchmarks averaged over three trials, are presented.
	Results of previously proposed methods are from \cite{liexploring}.}
	\label{table:baseline-results}
	\begin{adjustbox}{width=0.7\linewidth}
	\begin{tabular}{l|cccc|c|c}
		\toprule
		Algorithm & PACS & VLCS & OfficeHome & {TerraInc} & {DomainNet} & {Avg.}  \\
		
		\midrule
		
		ERM \cite{vapnik1998statistical}   & 
		95.9\scriptsize            & 
		81.9\scriptsize             & 
		84.1\scriptsize             & 
		56.1\scriptsize             & 
		59.5\scriptsize             & 75.5  \\
		
		\midrule
		
		IRM \cite{arjovsky2019irm} &
		96.1 &
		82.9 &
		83.2 & 
		56.7 &
		59.1 &
		75.6 \\
		
		DANN \cite{ganin2016domain} & 
		96.3 &
		81.8 &
		83.0 & 
		56.0 &
		58.4 &
		75.1 \\
		
		CDANN \cite{li2018cdann}    & 
		96.5 &
		82.4 & 
		82.9 &
		55.6 &
		58.4 &
		75.2 \\
		
		MMD \cite{li2018mmd}  &
		95.8 &
		82.3 &
		83.6 &
		57.4 & 
		59.9 &
		75.8 \\
		
		RSC \cite{huang2020rsc}               & 
		96.5         & 
		82.2                              & 
		83.2           & 
		58.2            & 
		59.0          & 
		75.8 \\    
	
		CORAL \cite{sun2016coral} & 
		96.4 &
		82.6 &
		83.8 &
		57.5 &
		59.8 &
		76.0 \\

		IIB \cite{li2022invariant} & 
		96.5 &
		82.3 & 
		84.2 &
		58.2 &
		58.6 &
		76.0 \\
		
		FISH \cite{shi2021gradient} & 
		96.9 &
		82.7 &
		85.0 & 
		58.0 &
		59.1 &
		75.0 \\
		
		SAM \cite{foret2020sharpness}  & 
		96.6         & 
		82.9         & 
		85.4        & 
		56.2         & 
		59.8               
		&  76.2 \\   

		GSAM \cite{zhuang2022surrogate}  & 
		96.6        & 
		82.9          & 
		{85.6}           & 
		55.4           & 
		59.8              &   
		76.1 \\
		
		GAM \cite{zhang2023gradient}     &
		{96.4}      & 
		{83.6}       & 
		{85.5}       & 
		{55.3}       & 
		59.5        &  
		76.1 \\

		SAGM \cite{wang2023sharpness}       & 
		{96.8}      & 
		{82.8}       & 
		{85.2}       & 
		{58.0}       & 
		59.5		&
		76.4 \\
		
		GGA\cite{ballas2025gradient}       & 
		 \underline{97.2}      & 
		 84.1  & 
		 85.7  & 
		 58.8  & 
	     59.8  &  
		 77.1 \\
		
		DISAM \cite{zhang2024domaininspired} &
		97.1 &
		82.7 &
		85.4 &
		57.3 &
		{59.8} &
		76.5 \\
		
		\midrule
		
		BOA (SAM pretrained) \cite{liexploring} &
		\textbf{97.4} & 
		\textbf{86.5} &
		\underline{86.0} & 
		\underline{60.3} &
		\underline{60.2} &
		\underline{78.1} \\
		
		\midrule
		
		\textbf{SAGE} (Ours)               & 
		\textbf{97.4}     & 
		\underline{85.5}  & 
		\textbf{87.1}            &  
		\textbf{61.1}            & 
		\textbf{63.4}            & 
		\textbf{78.9}  \\
		
		\bottomrule
		
	\end{tabular}
\end{adjustbox}
\end{table*}



\begin{table*}
  \centering
  	\caption{Results on NYU-v2 (3-task) and Cityscapes (2-task) datasets. The best results are highlighted in \textbf{bold} and second-best are \underline{underlined}. \textcolor{teal}{Green} plus signs (\textcolor{teal}{+}) and \textcolor{red}{Red} minuses (\textcolor{red}{$-$}) indicate whether the addition of SAGE on top of base solvers improves or hinders training, respectively.}
	\begin{adjustbox}{width=\linewidth}
		\begin{tabular}{l|llllllllll||lllll}
			\toprule
			\multicolumn{11}{c}{\textbf{~~~~~~~~~~~~~~~~~~~~~~~~NYU-v2}} & \multicolumn{5}{c}{\textbf{Cityscapes}} \\
			\multirow{3}*{Method} & \multicolumn{2}{c}{Segmentation} & \multicolumn{2}{c}{Depth} & \multicolumn{5}{c}{Surface Normal} & \multirow{3}*{$\Delta m\%\downarrow$} & \multicolumn{2}{c}{Segmentation} & \multicolumn{2}{c}{Depth} & \multirow{3}*{$\Delta m\%\downarrow$} \\
			\cmidrule(lr){2-3}\cmidrule(lr){4-5}\cmidrule(lr){6-10}\cmidrule(lr){12-13}\cmidrule(lr){14-15}
			& \multirow{2}*{mIoU $\uparrow$} & \multirow{2}*{Pix Acc $\uparrow$} & \multirow{2}*{Abs Err $\downarrow$} & \multirow{2}*{Rel Err $\downarrow$} & \multicolumn{2}{c}{Angle Distance $\downarrow$} & \multicolumn{3}{c}{Within $t^\circ$ $\uparrow$} & & \multirow{2}*{mIoU $\uparrow$} & \multirow{2}*{Pix Acc $\uparrow$} & \multirow{2}*{Abs Err $\downarrow$} & \multirow{2}*{Rel Err $\downarrow$} & \\
			\cmidrule(lr){6-7}\cmidrule(lr){8-10}
			& & & & & Mean & Median & $<$11.25 & $<$22.5 & $<$30 & & & & & & \\
			\midrule
			STL & 38.30 & 63.76 & 0.6754 & 0.2780 & 25.01 & 19.21 & 30.14 & 57.20 & 69.15 & & 74.01 & 93.16 & 0.0125 & 27.77 & \\
			\midrule
			LS & 39.29 & 65.33 & 0.5493 & 0.2263 & 28.15 & 23.96 & 22.09 & 47.50 & 61.08 & 5.59 & 75.18 & 93.49 & 0.0155 & 46.77 & 22.60 \\
			SI & 38.45 & 64.27 & 0.5354 & 0.2201 & 27.60 & 23.37 & 22.53 & 48.57 & 62.32 & 4.39 & 70.95 & 91.73 & 0.0161 & 33.83 & 14.11 \\
			RLW~\cite{lin2021reasonable} & 37.17 & 63.77 & 0.5759 & 0.2410 & 28.27 & 24.18 & 22.26 & 47.05 & 60.62 & 7.78 & 74.57 & 93.41 & 0.0158 & 47.79 & 24.38 \\
			DWA~\cite{liu2019end} & 39.11 & 65.31 & 0.5510 & 0.2285 & 27.61 & 23.18 & 24.17 & 50.18 & 62.39 & 3.57 & 75.24 & 93.52 & 0.0160 & 44.37 & 21.45 \\
			UW~\cite{kendall2018multi} & 36.87 & 63.17 & 0.5446 & 0.2260 & 27.04 & 22.61 & 23.54 & 49.05 & 63.65 & 4.05 & 72.02 & 92.85 & 0.0140 & 30.13 & 5.89 \\
			MGDA~\cite{sener2018multi} & 30.47 & 59.90 & 0.6070 & 0.2555 & 24.88 & 19.45 & 29.18 & 56.88 & 69.36 & 1.38 & 68.84 & 91.54 & 0.0309 & 33.50 & 44.14 \\
			PCGrad~\cite{yu2020gradient} & 38.06 & 64.64 & 0.5550 & 0.2325 & 27.41 & 22.80 & 23.86 & 49.83 & 63.14 & 3.97 & 75.13 & 93.48 & 0.0154 & 42.07 & 18.29 \\
			GradDrop~\cite{chen2020just} & 39.39 & 65.12 & 0.5455 & 0.2279 & 27.48 & 22.96 & 23.38 & 49.44 & 62.87 & 3.58 & 75.27 & 93.53 & 0.0157 & 47.54 & 23.73 \\
			IMTL-G~\cite{liu2021towards} & 39.35 & 65.60 & 0.5426 & 0.2256 & 26.02 & 21.19 & 26.20 & 53.13 & 66.24 & -0.76 & 75.33 & 93.49 & 0.0135 & 38.41 & 11.10 \\
			CAGrad~\cite{liu2021conflict} & 39.79 & 65.49 & 0.5486 & 0.2250 & 26.31 & 21.58 & 25.61 & 52.36 & 65.58 & 0.20 & 75.16 & 93.48 & 0.0141 & 37.60 & 11.64 \\
			MoCo~\cite{fernando2023mitigating} & 40.30 & \textbf{66.07} & 0.5575 & 0.2135 & 26.67 & 21.83 & 25.61 & 51.78 & 64.85 & 0.16 & 75.42 & 93.55 & 0.0149 & 34.19 & 9.90 \\
			Nash-MTL~\cite{navon2022multi} & 40.13 & 65.93 & \textbf{0.5261} & 0.2171 & 25.26 & 20.08 & 28.40 & 55.47 & 68.15 & -4.04 & 75.41 & 93.66 & 0.0129 & 35.02 & 6.82 \\
			FAMO~\cite{liu2023famo} & 38.88 & 64.90 & 0.5474 & 0.2194 & 25.06 & 19.57 & 29.21 & 56.61 & 68.98 & -4.10 & 74.54 & 93.29 & 0.0145 & 32.59 & 8.13 \\
			FairGrad~\cite{10.5555/3692070.3692179} & 38.80 & 65.29 & 0.5572 & 0.2322 & 24.55 & 18.97 & 30.50 & 57.94 & 70.14 & \underline{-4.96} & 74.10 & 93.03 & 0.0135 & 29.92 & 3.90 \\
			F-MTL~\cite{phan2025beyond} & \textbf{40.42} & 65.61 & 0.5389 & \underline{0.2121} & 25.03 & 19.75 & 28.90 & 56.19 & 68.72 & -4.77 & 76.63 & 93.76 & 0.0124 & 31.17 & \underline{1.87} \\
			\midrule

			SAMO-LS~\cite{ban2025samo} & 39.59 & \underline{65.72} & 0.5514 & 0.2246 & 27.38 & 22.78 & 24.09 & 49.82 & 63.01 & 2.88 & \textbf{76.46} & \textbf{93.76} & 0.0147 & 39.85 & 14.30 \\
			SAMO-MGDA~\cite{ban2025samo} & 29.85 & 60.83 & 0.6111 & 0.2388 & \textbf{24.11} & \underline{18.18} & \underline{32.16} & \underline{59.59} & \underline{71.15} & -2.19 & 73.28 & 93.26 & 0.0133 & 30.57 & 4.30 \\
			SAMO-FairGrad~\cite{ban2025samo} & 39.05 & 65.06 & 0.5359 & 0.2137 & 24.43 & 18.79 & 30.98 & 58.35 & 70.42 & \textbf{-6.55} & 74.37 & 93.14 & \underline{0.0129} & \textbf{26.30} & \textbf{-0.62} \\
			
			\midrule
			
			SAGE-LS & 39.67~\textbf{\textcolor{teal}{+}} & 65.03~\textbf{\textcolor{red}{$-$}} & 
			0.5331~\textbf{\textcolor{teal}{+}} & 
			0.2180~\textbf{\textcolor{teal}{+}} & 
			27.40~\textbf{\textcolor{teal}{+}} & 
			23.06~\textbf{\textcolor{teal}{+}} & 
			23.62~\textbf{\textcolor{teal}{+}} & 
			49.28~\textbf{\textcolor{teal}{+}} & 
			62.87~\textbf{\textcolor{teal}{+}} & 
			3.81~\textbf{\textcolor{teal}{+}} & 
			75.69~\textbf{\textcolor{teal}{+}} & 
			93.57~\textbf{\textcolor{teal}{+}} & 
			0.0146~\textbf{\textcolor{teal}{+}} & 
			42.84~\textbf{\textcolor{teal}{+}} & 
			19.92~\textbf{\textcolor{teal}{+}} \\
			
			SAGE-MGDA & 33.24~\textbf{\textcolor{teal}{+}} & 
			61.88~\textbf{\textcolor{teal}{+}} & 
			0.5707~\textbf{\textcolor{teal}{+}} & 
			0.2266~\textbf{\textcolor{teal}{+}} & 
			24.55~\textbf{\textcolor{teal}{+}} & 
			18.99~\textbf{\textcolor{teal}{+}} & 
			30.15~\textbf{\textcolor{teal}{+}} & 
			57.91~\textbf{\textcolor{teal}{+}} & 
			70.07~\textbf{\textcolor{teal}{+}} & 
			-1.77~\textbf{\textcolor{teal}{+}} & 
			73.46~\textbf{\textcolor{teal}{+}} & 
			93.28~\textbf{\textcolor{teal}{+}} & 
			\textbf{0.0123}~\textbf{\textcolor{teal}{+}} & 
			29.59~\textbf{\textcolor{teal}{+}} & 
			3.82~\textbf{\textcolor{teal}{+}} \\
			
			SAGE-FairGrad & 
			\underline{40.32}~\textbf{\textcolor{teal}{+}} & 
			65.34~\textbf{\textcolor{teal}{+}} & 
			\underline{0.5301}~\textbf{\textcolor{teal}{+}} & 
			\textbf{0.2082}~\textbf{\textcolor{teal}{+}} & 
			24.53~\textbf{\textcolor{teal}{+}} & 
			18.97~\textbf{{$\pm$}} & 
			30.28~\textbf{\textcolor{red}{$-$}} & 
			57.44~\textbf{\textcolor{red}{$-$}} & 
			69.66~\textbf{\textcolor{red}{$-$}} & 
			\underline{-5.63}~\textbf{\textcolor{teal}{+}}& 
			75.40~\textbf{\textcolor{teal}{+}} & 
			93.63~\textbf{\textcolor{teal}{+}} & 
			\underline{0.0129}~\textbf{\textcolor{teal}{+}} & 
			29.33~\textbf{\textcolor{teal}{+}} & 
			{2.63}~\textbf{\textcolor{teal}{+}} \\
			
			SAGE-SAMO-LS & 
			40.15~\textbf{\textcolor{teal}{+}} & 
			65.59~\textbf{\textcolor{red}{$-$}} & 
			0.5508~\textbf{\textcolor{teal}{+}} & 
			0.2251~\textbf{\textcolor{red}{$-$}} & 
			27.40~\textbf{\textcolor{red}{$-$}} & 
			22.76~\textbf{\textcolor{teal}{+}} & 
			23.34~\textbf{\textcolor{red}{$-$}} & 
			49.86~\textbf{\textcolor{teal}{+}} & 
			63.56~\textbf{\textcolor{teal}{+}} & 
			3.82~\textbf{\textcolor{red}{$-$}} & 
			\underline{76.07}~\textbf{\textcolor{red}{$-$}} & 
			\underline{93.73}~\textbf{\textcolor{red}{$-$}} & 
			0.0149~\textbf{\textcolor{red}{$-$}} & 
			42.15~\textbf{\textcolor{red}{$-$}} & 
			18.71~\textbf{\textcolor{red}{$-$}} \\
			
			SAGE-SAMO-MGDA & 
			32.32~\textbf{\textcolor{teal}{+}} & 
			62.78~\textbf{\textcolor{teal}{+}} & 
			0.5651~\textbf{\textcolor{teal}{+}} & 
			0.2201~\textbf{\textcolor{teal}{+}} & 
			\underline{24.23}~\textbf{\textcolor{red}{$-$}} & 
			\textbf{18.10}~\textbf{\textcolor{teal}{+}} & 
			\textbf{32.41}~\textbf{\textcolor{teal}{+}} & 
			\textbf{59.68}~\textbf{\textcolor{teal}{+}} & 
			\textbf{71.16}~\textbf{\textcolor{teal}{+}} & 
			-3.66~\textbf{\textcolor{teal}{+}} & 
			73.92~\textbf{\textcolor{teal}{+}} & 
			93.20~\textbf{\textcolor{red}{$-$}}& 
			0.0133~\textbf{$\pm$} & 
			32.46~\textbf{\textcolor{red}{$-$}} &
			9.85~\textbf{\textcolor{red}{$-$}} \\
			
			SAGE-SAMO-FairGrad & 
			39.76~\textbf{\textcolor{teal}{+}} & 
			65.80~\textbf{\textcolor{teal}{+}} & 
			0.5413~\textbf{\textcolor{red}{$-$}} & 
			0.2164~\textbf{\textcolor{red}{$-$}} & 
			24.34~\textbf{\textcolor{teal}{+}} & 
			18.69~\textbf{\textcolor{teal}{+}} & 
			31.01~\textbf{\textcolor{teal}{+}} & 
			58.38~\textbf{\textcolor{teal}{+}} & 
			70.46~\textbf{\textcolor{teal}{+}} & 
			\underline{-6.19}~\textbf{\textcolor{red}{$-$}} & 
			75.08~\textbf{\textcolor{teal}{+}}  & 
			93.52~\textbf{\textcolor{teal}{+}}  & 
			0.0131~\textbf{\textcolor{red}{$-$}} & 
			32.02~\textbf{\textcolor{red}{$-$}} & 
			4.88~\textbf{\textcolor{red}{$-$}} \\
			
			\bottomrule
		\end{tabular}
	\end{adjustbox}
	\vspace{0.1em}
	\label{tab:combined}
\end{table*}


%
%

\paragraph{Summary.} Across the two benchmark settings, SAGE establishes a new state-of-the-art on the DomainBed DG benchmark (78.9\% average) and acts as a general-purpose improvement for base MTL solvers. On MTL specifically, specialized methods such as SAMO-FairGrad remain the strongest on absolute $\Delta m$\%, while SAGE provides consistent gains over the underlying solvers (LS, MGDA, FairGrad) without requiring task-routing or solver-specific design. In combination, the DG experiments demonstrate that targeting both curvature 
and gradient alignment can yield state-of-the-art results, while the MTL results
demonstrate that the same mechanisms can be added to existing gradient-balancing methods and yield improved results.

\subsection{Ablations}
Figure~\ref{fig:ablation-comparison} (left) reports OfficeHome accuracy over 
selections of the perturbation radius $\rho$ and the noise scale $\gamma$. 
Performance is stable across selection, with the exception of the selection of 
either relatively low or high values. In the same, figure on the right we compare the performance of SAGE when removing each of its components and against the Muon~\cite{jordan2024muon} optimizer. As apparent from the results, 
each component of SAGE, i.e. orthogonalization via Newton-Schulz, weight 
scaling in each layer and the additive noise, contributes positively during 
training. Further ablations, such as cosine similarity plots between the 
average training-domain and unseen target domain gradients are provided in 
Appendix~\ref{app:ablations}.

\begin{figure}[htbp]
	\centering
	
	\begin{minipage}[t]{0.39\linewidth}
		\vspace{0pt}
		\centering
		\includegraphics[width=\linewidth]{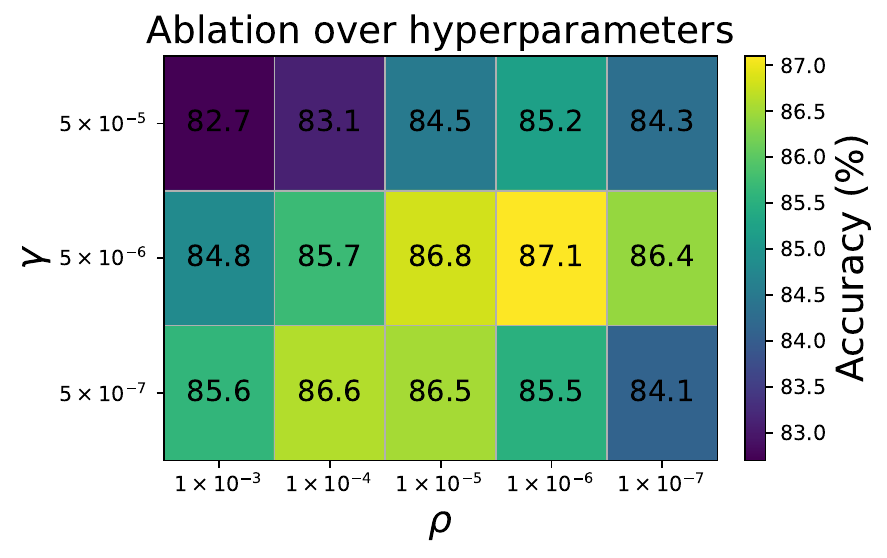}
	\end{minipage}
	\hfill
	\begin{minipage}[t]{0.59\linewidth}
		\vspace{0pt}
		\centering
		\label{tab:ablation_sage}
		
		\begin{adjustbox}{max width=\linewidth}
			\begin{tabular}{c|c}
				\toprule
				Method & Accuracy\\
				\midrule
				No Newton-Schulz & 86.3 \\
				No Scaling & 86.6 \\
				No Noise & 86.8 \\
				Muon~\cite{jordan2024muon} & 85.5 \\
				\midrule
				SAGE & 87.1 \\
				\bottomrule
			\end{tabular}
		\end{adjustbox}
	\end{minipage}
	
	\caption{\textit{Left}: Ablation results for the $\rho$ and $\gamma$ hyperparameters and \textit{right}: Ablations regarding each component of SAGE, on the OfficeHome dataset.}
	\label{fig:ablation-comparison}
\end{figure}

\section{Conclusions, Limitations \& Future Work}
\label{sec:conclusion}

In this paper, we presented an excess-risk decomposition for multi-distribution 
learning that exposes a structural trade-off between the curvature of the loss 
surface and the alignment of per-distribution gradients, along with a 
separation result showing that neither quantity bounds the other. Motivated by this decomposition, we proposed SAGE, whose two components target the two terms via: a spectral perturbation that probes all singular directions of each 
layer's gradient with comparable magnitude, and a gradient-disagreement-scaled 
noise injection that biases the optimizer away from high-conflict regions. Across five DG and two MTL benchmarks, SAGE compares favorably with both 
sharpness-aware and gradient-alignment baselines.

One limitation of SAGE is that it incurs additional per-step cost relative to
vanilla SAM, as Newton–Schulz orthogonalization adds matrix multiplications on
every weight matrix, and the noise mechanism requires per-distribution gradient
computation (Appendix~\ref{app:complexity}). Additionally, the empirical gains
are not uniformly state-of-the-art across all benchmarks. In the MTL setting in
particular, when paired with base solvers SAGE improves performance in almost
all cases, but does not dominate every metric compared to methods that have been
specifically developed for MTL. More fundamentally, the connection between
SAGE's mechanisms and the terms of Theorem~\ref{thm:decomposition} is
motivational rather than formal: we have not shown that either component
provably reduces a tractable surrogate of its target term, and
Counterexample~\ref{thm:decoupling} is an argument in the quadratic family
rather than a guarantee for the non-convex deep-learning regime. These
limitations point to natural directions for future work. For example, methods
that directly estimate or bound the two terms of Theorem~\ref{thm:decomposition}
could replace the spectral perturbation with a target that more closely matches
the curvature term, while estimators of $\Sigma_g$ could be coupled with
regularizers that penalize the alignment term more directly. Finally, whether
the curvature–alignment trade-off identified also applies to other settings such
as continual learning or federated optimization, where multi-distribution
structure arises from different generative processes than DG and MTL, is an
interesting open question which we aim to research in future work.

%


\bibliography{main}

@String(CVPR= {IEEE Conf. Comput. Vis. Pattern Recog.})

@String(ECCV= {Eur. Conf. Comput. Vis.})

@String(ICASSP=	{ICASSP})

@String(AAAI = {AAAI})

@String(CVPR  = {CVPR})

@String(ECCV  = {ECCV})

@article{zhou_domain_2023,
	title = {Domain {Generalization}: {A} {Survey}},
	volume = {45},
	issn = {1939-3539},
	shorttitle = {Domain {Generalization}},
	doi = {10.1109/TPAMI.2022.3195549},
	number = {4},
	journal = {IEEE Transactions on Pattern Analysis and Machine Intelligence},
	author = {Zhou, Kaiyang and Liu, Ziwei and Qiao, Yu and Xiang, Tao and Loy, Chen Change},
	month = apr,
	year = {2023},
	note = {Conference Name: IEEE Transactions on Pattern Analysis and Machine Intelligence},
	pages = {4396--4415},
}

@article{yu2020gradient,
	title={Gradient surgery for multi-task learning},
	author={Yu, Tianhe and Kumar, Saurabh and Gupta, Abhishek and Levine, Sergey and Hausman, Karol and Finn, Chelsea},
	journal={Advances in Neural Information Processing Systems},
	volume={33},
	pages={5824--5836},
	year={2020}
}

@inproceedings{li2018mmd,
	title={Domain generalization with adversarial feature learning},
	author={Li, Haoliang and Pan, Sinno Jialin and Wang, Shiqi and Kot, Alex C},
	booktitle={Computer Vision and Pattern Recognition},
	year={2018}
}

@article{arjovsky2019irm,
	title={Invariant risk minimization},
	author={Arjovsky, Martin and Bottou, L{\'e}on and Gulrajani, Ishaan and Lopez-Paz, David},
	journal={arXiv preprint arXiv:1907.02893},
	year={2019}
}

@article{ganin2016domain,
	title={Domain-adversarial training of neural networks},
	author={Ganin, Yaroslav and Ustinova, Evgeniya and Ajakan, Hana and Germain, Pascal and Larochelle, Hugo and Laviolette, Fran{\c{c}}ois and March, Mario and Lempitsky, Victor},
	journal={Journal of machine learning research},
	volume={17},
	number={59},
	pages={1--35},
	year={2016}
}

@inproceedings{li2018cdann,
	title={Domain generalization via conditional invariant representations},
	author={Li, Ya and Gong, Mingming and Tian, Xinmei and Liu, Tongliang and Tao, Dacheng},
	booktitle={AAAI Conference on Artificial Intelligence},
	volume={32},
	year={2018}
}

@article{huang2020rsc,
	title={Self-challenging improves cross-domain generalization},
	author={Huang, Zeyi and Wang, Haohan and Xing, Eric P and Huang, Dong},
	journal={European Conference on Computer Vision},
	year={2020},
}

@article{vapnik1998statistical,
	title={Statistical Learning Theory},
	author={Vapnik, V},
	journal={NY: Wiley},
	year={1998}
}

@inproceedings{
	shi2021gradient,
	title={Gradient Matching for Domain Generalization},
	author={Yuge Shi and Jeffrey Seely and Philip Torr and Siddharth N and Awni Hannun and Nicolas Usunier and Gabriel Synnaeve},
	booktitle={International Conference on Learning Representations},
	year={2022},
	url={https://openreview.net/forum?id=vDwBW49HmO}
}

@inproceedings{sun2016coral,
	title={Deep coral: Correlation alignment for deep domain adaptation},
	author={Sun, Baochen and Saenko, Kate},
	booktitle={European Conference on Computer Vision},
	year={2016},
}

@inproceedings{foret2020sharpness,
	title={Sharpness-Aware Minimization for Efficiently Improving Generalization},
	author={Foret, Pierre and Kleiner, Ariel and Mobahi, Hossein and Neyshabur, Behnam},
	booktitle={International Conference on Learning Representations},
	year={2021}
}

@inproceedings{
	zhuang2022surrogate,
	title={Surrogate Gap Minimization Improves Sharpness-Aware Training},
	author={Juntang Zhuang and Boqing Gong and Liangzhe Yuan and Yin Cui and Hartwig Adam and Nicha C Dvornek and sekhar tatikonda and James s Duncan and Ting Liu},
	booktitle={International Conference on Learning Representations},
	year={2022},
	url={https://openreview.net/forum?id=edONMAnhLu-}
}

@inproceedings{wang2023sharpness,
	title={Sharpness-aware gradient matching for domain generalization},
	author={Wang, Pengfei and Zhang, Zhaoxiang and Lei, Zhen and Zhang, Lei},
	booktitle={Proceedings of the IEEE/CVF Conference on Computer Vision and Pattern Recognition},
	pages={3769--3778},
	year={2023}
}

@inproceedings{le2024gradient,
	title={Gradient alignment for cross-domain face anti-spoofing},
	author={Le, Binh M and Woo, Simon S},
	booktitle={Proceedings of the IEEE/CVF Conference on Computer Vision and Pattern Recognition},
	pages={188--199},
	year={2024}
}

@inproceedings{gulrajani2020domainbed,
	title={In search of lost domain generalization},
	author={Gulrajani, Ishaan and Lopez-Paz, David},
	booktitle={International Conference on Learning Representations},
	year={2021}
}

@inproceedings{mansilla2021domain,
	title={Domain generalization via gradient surgery},
	author={Mansilla, Lucas and Echeveste, Rodrigo and Milone, Diego H and Ferrante, Enzo},
	booktitle={Proceedings of the IEEE/CVF international conference on computer vision},
	pages={6630--6638},
	year={2021}
}

@inproceedings{li2017deeper,
	title={Deeper, broader and artier domain generalization},
	author={Li, Da and Yang, Yongxin and Song, Yi-Zhe and Hospedales, Timothy M},
	booktitle={Proceedings of the IEEE international conference on computer vision},
	pages={5542--5550},
	year={2017}
}

@inproceedings{fang2013unbiased,
	title={Unbiased metric learning: On the utilization of multiple datasets and web images for softening bias},
	author={Fang, Chen and Xu, Ye and Rockmore, Daniel N},
	booktitle={Proceedings of the IEEE International Conference on Computer Vision},
	pages={1657--1664},
	year={2013}
}

@inproceedings{venkateswara2017deep,
	title={Deep hashing network for unsupervised domain adaptation},
	author={Venkateswara, Hemanth and Eusebio, Jose and Chakraborty, Shayok and Panchanathan, Sethuraman},
	booktitle={Proceedings of the IEEE conference on computer vision and pattern recognition},
	pages={5018--5027},
	year={2017}
}

@inproceedings{beery2018recognition,
	title={Recognition in terra incognita},
	author={Beery, Sara and Van Horn, Grant and Perona, Pietro},
	booktitle={Proceedings of the European conference on computer vision (ECCV)},
	pages={456--473},
	year={2018}
}

@inproceedings{peng2019moment,
	title={Moment matching for multi-source domain adaptation},
	author={Peng, Xingchao and Bai, Qinxun and Xia, Xide and Huang, Zijun and Saenko, Kate and Wang, Bo},
	booktitle={Proceedings of the IEEE/CVF international conference on computer vision},
	pages={1406--1415},
	year={2019}
}

@article{cha2021swad,
	title={Swad: Domain generalization by seeking flat minima},
	author={Cha, Junbum and Chun, Sanghyuk and Lee, Kyungjae and Cho, Han-Cheol and Park, Seunghyun and Lee, Yunsung and Park, Sungrae},
	journal={Advances in Neural Information Processing Systems},
	volume={34},
	pages={22405--22418},
	year={2021}
}

@inproceedings{kendall2018multi,
	title={Multi-task learning using uncertainty to weigh losses for scene geometry and semantics},
	author={Kendall, Alex and Gal, Yarin and Cipolla, Roberto},
	booktitle={Proceedings of the IEEE conference on computer vision and pattern recognition},
	pages={7482--7491},
	year={2018}
}

@article{sener2018multi,
	title={Multi-task learning as multi-objective optimization},
	author={Sener, Ozan and Koltun, Vladlen},
	journal={Advances in neural information processing systems},
	volume={31},
	year={2018}
}

@ARTICLE{9392366,
	author={Zhang, Yu and Yang, Qiang},
	journal={IEEE Transactions on Knowledge and Data Engineering}, 
	title={A Survey on Multi-Task Learning}, 
	year={2022},
	volume={34},
	number={12},
	pages={5586-5609},
	doi={10.1109/TKDE.2021.3070203}}

@article{wang2022generalizing,
	title={Generalizing to unseen domains: A survey on domain generalization},
	author={Wang, Jindong and Lan, Cuiling and Liu, Chang and Ouyang, Yidong and Qin, Tao and Lu, Wang and Chen, Yiqiang and Zeng, Wenjun and Philip, S Yu},
	journal={IEEE transactions on knowledge and data engineering},
	volume={35},
	number={8},
	pages={8052--8072},
	year={2022},
	publisher={IEEE}
}

@ARTICLE{10233054,
	author={Ballas, Aristotelis and Diou, Christos},
	journal={IEEE Transactions on Emerging Topics in Computational Intelligence}, 
	title={Towards Domain Generalization for ECG and EEG Classification: Algorithms and Benchmarks}, 
	year={2024},
	volume={8},
	number={1},
	pages={44-54},
	doi={10.1109/TETCI.2023.3306253}}

@inproceedings{krueger2021out,
	title={Out-of-distribution generalization via risk extrapolation (rex)},
	author={Krueger, David and Caballero, Ethan and Jacobsen, Joern-Henrik and Zhang, Amy and Binas, Jonathan and Zhang, Dinghuai and Le Priol, Remi and Courville, Aaron},
	booktitle={International conference on machine learning},
	pages={5815--5826},
	year={2021},
	organization={PMLR}
}

@inproceedings{welling2011bayesian,
	title={Bayesian learning via stochastic gradient Langevin dynamics},
	author={Welling, Max and Teh, Yee W},
	booktitle={Proceedings of the 28th international conference on machine learning (ICML-11)},
	pages={681--688},
	year={2011}
}

@inproceedings{zhang2024domaininspired,
	title={Domain-Inspired Sharpness Aware Minimization Under Domain Shifts},
	author={Ruipeng Zhang and Ziqing Fan and Jiangchao Yao and Ya Zhang and Yanfeng Wang},
	booktitle={The Twelfth International Conference on Learning Representations},
	year={2024},
	url={https://openreview.net/forum?id=I4wB3HA3dJ}
}

@inproceedings{dinh2017sharp,
	title={Sharp minima can generalize for deep nets},
	author={Dinh, Laurent and Pascanu, Razvan and Bengio, Samy and Bengio, Yoshua},
	booktitle={International Conference on Machine Learning},
	pages={1019--1028},
	year={2017},
	organization={PMLR}
}

@inproceedings{kwon2021asam,
	title={Asam: Adaptive sharpness-aware minimization for scale-invariant learning of deep neural networks},
	author={Kwon, Jungmin and Kim, Jeongseop and Park, Hyunseo and Choi, In Kwon},
	booktitle={International conference on machine learning},
	pages={5905--5914},
	year={2021},
	organization={PMLR}
}

@inproceedings{ballas2025gradient,
	title={Gradient-guided annealing for domain generalization},
	author={Ballas, Aristotelis and Diou, Christos},
	booktitle={Proceedings of the Computer Vision and Pattern Recognition Conference},
	pages={20558--20568},
	year={2025}
}

@misc{jordan2024muon,
	author       = {Keller Jordan and Yuchen Jin and Vlado Boza and You Jiacheng and
	Franz Cesista and Laker Newhouse and Jeremy Bernstein},
	title        = {Muon: An optimizer for hidden layers in neural networks},
	year         = {2024},
	url          = {https://kellerjordan.github.io/posts/muon/}
}

@article{kovarik1970some,
	title={Some iterative methods for improving orthonormality},
	author={Kovarik, Zdislav},
	journal={SIAM Journal on Numerical Analysis},
	volume={7},
	number={3},
	pages={386--389},
	year={1970},
	publisher={SIAM}
}

@article{bjorck1971iterative,
	title={An iterative algorithm for computing the best estimate of an orthogonal matrix},
	author={Bj{\"o}rck, {\AA}ke and Bowie, Clazett},
	journal={SIAM Journal on Numerical Analysis},
	volume={8},
	number={2},
	pages={358--364},
	year={1971},
	publisher={SIAM}
}

@article{guo2006schur,
	title={A schur--newton method for the matrix$\backslash$boldmath p th root and its inverse},
	author={Guo, Chun-Hua and Higham, Nicholas J},
	journal={SIAM Journal on Matrix Analysis and Applications},
	volume={28},
	number={3},
	pages={788--804},
	year={2006},
	publisher={SIAM}
}

@article{bernstein2024old,
	title={Old optimizer, new norm: An anthology},
	author={Bernstein, Jeremy and Newhouse, Laker},
	journal={arXiv preprint arXiv:2409.20325},
	year={2024}
}

@article{liu2021conflict,
	title={Conflict-averse gradient descent for multi-task learning},
	author={Liu, Bo and Liu, Xingchao and Jin, Xiaojie and Stone, Peter and Liu, Qiang},
	journal={Advances in neural information processing systems},
	volume={34},
	pages={18878--18890},
	year={2021}
}

@inproceedings{
	shin2025sassha,
	title={Sassha: Sharpness-aware Adaptive Second-order Optimization with Stable Hessian Approximation},
	author={Dahun Shin and Dongyeop Lee and Jinseok Chung and Namhoon Lee},
	booktitle={Forty-second International Conference on Machine Learning},
	year={2025},
	url={https://openreview.net/forum?id=7bgqx5OoVe}
}

@inproceedings{zhang2023gradient,
	title={Gradient norm aware minimization seeks first-order flatness and improves generalization},
	author={Zhang, Xingxuan and Xu, Renzhe and Yu, Han and Zou, Hao and Cui, Peng},
	booktitle={Proceedings of the IEEE/CVF Conference on Computer Vision and Pattern Recognition},
	pages={20247--20257},
	year={2023}
}

@inproceedings{liexploring,
	title={Exploring Mode Connectivity in Krylov Subspace for Domain Generalization},
	author={Li, Aodi and Zhuang, Liansheng and Long, Xiao and Li, Houqiang and Wang, Shafei},
	booktitle={The Fourteenth International Conference on Learning Representations},
	year={2026}
}

@inproceedings{li2022invariant,
	title={Invariant information bottleneck for domain generalization},
	author={Li, Bo and Shen, Yifei and Wang, Yezhen and Zhu, Wenzhen and Li, Dongsheng and Keutzer, Kurt and Zhao, Han},
	booktitle={Proceedings of the AAAI Conference on Artificial Intelligence},
	volume={36},
	pages={7399--7407},
	year={2022}
}

@article{li2018visualizing,
	title={Visualizing the loss landscape of neural nets},
	author={Li, Hao and Xu, Zheng and Taylor, Gavin and Studer, Christoph and Goldstein, Tom},
	journal={Advances in neural information processing systems},
	volume={31},
	year={2018}
}

@article{wu2017towards,
	title={Towards understanding generalization of deep learning: Perspective of loss landscapes},
	author={Wu, Lei and Zhu, Zhanxing and others},
	journal={arXiv preprint arXiv:1706.10239},
	year={2017}
}

@article{fort2019large,
	title={Large scale structure of neural network loss landscapes},
	author={Fort, Stanislav and Jastrzebski, Stanislaw},
	journal={Advances in Neural Information Processing Systems},
	volume={32},
	year={2019}
}

@phdthesis{rangamani2020loss,
	title={Loss landscapes and generalization in neural networks: Theory and applications},
	author={Rangamani, Akshay and others},
	year={2020},
	school={Johns Hopkins University}
}

@article{keskar2016large,
	title={On large-batch training for deep learning: Generalization gap and sharp minima},
	author={Keskar, Nitish Shirish and Mudigere, Dheevatsa and Nocedal, Jorge and Smelyanskiy, Mikhail and Tang, Ping Tak Peter},
	journal={arXiv preprint arXiv:1609.04836},
	year={2016}
}

@article{jiang2019fantastic,
	title={Fantastic generalization measures and where to find them},
	author={Jiang, Yiding and Neyshabur, Behnam and Mobahi, Hossein and Krishnan, Dilip and Bengio, Samy},
	journal={arXiv preprint arXiv:1912.02178},
	year={2019}
}

@inproceedings{kim2022fisher,
	title={Fisher sam: Information geometry and sharpness aware minimisation},
	author={Kim, Minyoung and Li, Da and Hu, Shell X and Hospedales, Timothy},
	booktitle={International Conference on Machine Learning},
	pages={11148--11161},
	year={2022},
	organization={PMLR}
}

@inproceedings{rame2022fishr,
	title={Fishr: Invariant gradient variances for out-of-distribution generalization},
	author={Rame, Alexandre and Dancette, Corentin and Cord, Matthieu},
	booktitle={International Conference on Machine Learning},
	pages={18347--18377},
	year={2022},
	organization={PMLR}
}

@article{teh2016consistency,
	title={Consistency and fluctuations for stochastic gradient Langevin dynamics},
	author={Teh, Yee Whye and Thi{\'e}ry, Alexandre and Vollmer, Sebastian J},
	journal={Journal of Machine Learning Research},
	volume={17},
	number={7},
	year={2016},
	publisher={Journal of Machine Learning Research}
}

@article{li2023enhancing,
	title={Enhancing sharpness-aware optimization through variance suppression},
	author={Li, Bingcong and Giannakis, Georgios},
	journal={Advances in Neural Information Processing Systems},
	volume={36},
	pages={70861--70879},
	year={2023}
}

@inproceedings{zhang2025preconditioned,
	title={Preconditioned sharpness-aware minimization: Unifying analysis and a novel learning algorithm},
	author={Zhang, Yilang and Li, Bingcong and Giannakis, Georgios B},
	booktitle={ICASSP 2025-2025 IEEE International Conference on Acoustics, Speech and Signal Processing (ICASSP)},
	pages={1--5},
	year={2025},
	organization={IEEE}
}

@inproceedings{ban2025samo,
	title={Samo: A lightweight sharpness-aware approach for multi-task optimization with joint global-local perturbation},
	author={Ban, Hao and Subramani, Gokul Ram and Ji, Kaiyi},
	booktitle={Proceedings of the IEEE/CVF International Conference on Computer Vision},
	pages={785--795},
	year={2025}
}

@InProceedings{Senushkin_2023_CVPR,
	author    = {Senushkin, Dmitry and Patakin, Nikolay and Kuznetsov, Arseny and Konushin, Anton},
	title     = {Independent Component Alignment for Multi-Task Learning},
	booktitle = {Proceedings of the IEEE/CVF Conference on Computer Vision and Pattern Recognition (CVPR)},
	month     = {June},
	year      = {2023},
	pages     = {20083-20093}
}

@book{van2000asymptotic,
	title={Asymptotic statistics},
	author={Van der Vaart, Aad W},
	volume={3},
	year={2000},
	publisher={Cambridge university press}
}

@misc{wen2023doessharpnessawareminimizationminimize,
	title={How Does Sharpness-Aware Minimization Minimize Sharpness?}, 
	author={Kaiyue Wen and Tengyu Ma and Zhiyuan Li},
	year={2023},
	eprint={2211.05729},
	archivePrefix={arXiv},
	primaryClass={cs.LG},
	url={https://arxiv.org/abs/2211.05729}, 
}

@inproceedings{gupta2018shampoo,
	title={Shampoo: Preconditioned stochastic tensor optimization},
	author={Gupta, Vineet and Koren, Tomer and Singer, Yoram},
	booktitle={International Conference on Machine Learning},
	pages={1842--1850},
	year={2018},
	organization={PMLR}
}

@inproceedings{navon2022multi,
	title={Multi-task learning as a bargaining game},
	author={Navon, Aviv and Shamsian, Aviv and Achituve, Idan and Maron, Haggai and Kawaguchi, Kenji and Chechik, Gal and Fetaya, Ethan},
	booktitle={International Conference on Machine Learning},
	pages={16428–-16446},
	year={2022},
	organization={PMLR}
}

@inproceedings{cordts2016cityscapes,
	title={The cityscapes dataset for semantic urban scene understanding},
	author={Cordts, Marius and Omran, Mohamed and Ramos, Sebastian and Rehfeld, Timo and Enzweiler, Markus and Benenson, Rodrigo and Franke, Uwe and Roth, Stefan and Schiele, Bernt},
	booktitle={Proceedings of the IEEE conference on computer vision and pattern recognition},
	pages={3213--3223},
	year={2016}
}

@inproceedings{silberman2012indoor,
	title={Indoor segmentation and support inference from rgbd images},
	author={Silberman, Nathan and Hoiem, Derek and Kohli, Pushmeet and Fergus, Rob},
	booktitle={European conference on computer vision},
	pages={746--760},
	year={2012},
	organization={Springer}
}

@inproceedings{radford2021learning,
	title={Learning transferable visual models from natural language supervision},
	author={Radford, Alec and Kim, Jong Wook and Hallacy, Chris and Ramesh, Aditya and Goh, Gabriel and Agarwal, Sandhini and Sastry, Girish and Askell, Amanda and Mishkin, Pamela and Clark, Jack and others},
	booktitle={International conference on machine learning},
	pages={8748--8763},
	year={2021},
	organization={PmLR}
}

@inproceedings{
	dosovitskiy2021an,
	title={An Image is Worth 16x16 Words: Transformers for Image Recognition at Scale},
	author={Alexey Dosovitskiy and Lucas Beyer and Alexander Kolesnikov and Dirk Weissenborn and Xiaohua Zhai and Thomas Unterthiner and Mostafa Dehghani and Matthias Minderer and Georg Heigold and Sylvain Gelly and Jakob Uszkoreit and Neil Houlsby},
	booktitle={International Conference on Learning Representations},
	year={2021},
	url={https://openreview.net/forum?id=YicbFdNTTy}
}

@article{liu2023famo,
	title={Famo: Fast adaptive multitask optimization},
	author={Liu, Bo and Feng, Yihao and Stone, Peter and Liu, Qiang},
	journal={Advances in Neural Information Processing Systems},
	volume={36},
	pages={57226--57243},
	year={2023}
}

@article{badrinarayanan2017segnet,
	title={Segnet: A deep convolutional encoder-decoder architecture for image segmentation},
	author={Badrinarayanan, Vijay and Kendall, Alex and Cipolla, Roberto},
	journal={IEEE transactions on pattern analysis and machine intelligence},
	volume={39},
	number={12},
	pages={2481--2495},
	year={2017},
	publisher={IEEE}
}

@article{lin2021reasonable,
	title={Reasonable effectiveness of random weighting: A litmus test for multi-task learning},
	author={Lin, Baijiong and Ye, Feiyang and Zhang, Yu and Tsang, Ivor W},
	journal={arXiv preprint arXiv:2111.10603},
	year={2021}
}

@inproceedings{liu2019end,
	title={End-to-end multi-task learning with attention},
	author={Liu, Shikun and Johns, Edward and Davison, Andrew J},
	booktitle={Proceedings of the IEEE/CVF conference on computer vision and pattern recognition},
	pages={1871--1880},
	year={2019}
}

@inproceedings{
	liu2021towards,
	title={Towards Impartial Multi-task Learning},
	author={Liyang Liu and Yi Li and Zhanghui Kuang and Jing-Hao Xue and Yimin Chen and Wenming Yang and Qingmin Liao and Wayne Zhang},
	booktitle={International Conference on Learning Representations},
	year={2021},
	url={https://openreview.net/forum?id=IMPnRXEWpvr}
}

@article{chen2020just,
	title={Just pick a sign: Optimizing deep multitask models with gradient sign dropout},
	author={Chen, Zhao and Ngiam, Jiquan and Huang, Yanping and Luong, Thang and Kretzschmar, Henrik and Chai, Yuning and Anguelov, Dragomir},
	journal={Advances in Neural Information Processing Systems},
	volume={33},
	pages={2039--2050},
	year={2020}
}

@inproceedings{fernando2023mitigating,
	title={Mitigating gradient bias in multi-objective learning: A provably convergent approach},
	author={Fernando, Heshan Devaka and Shen, Han and Liu, Miao and Chaudhury, Subhajit and Murugesan, Keerthiram and Chen, Tianyi},
	booktitle={The eleventh international conference on learning representations},
	year={2023}
}

@inproceedings{10.5555/3692070.3692179,
	author = {Ban, Hao and Ji, Kaiyi},
	title = {Fair resource allocation in multi-task learning},
	year = {2024},
	publisher = {JMLR.org},
	booktitle = {Proceedings of the 41st International Conference on Machine Learning},
	articleno = {109},
	numpages = {17},
	location = {Vienna, Austria},
	series = {ICML'24}
}

@inproceedings{phan2025beyond,
	title={Beyond Losses Reweighting: Empowering Multi-Task Learning via the Generalization Perspective},
	author={Phan, Hoang and Tran, Lam and Tran, Quyen and Tran, Ngoc and Truong, Tuan and Lei, Qi and Ho, Nhat and Phung, Dinh and Le, Trung},
	booktitle={Proceedings of the IEEE/CVF International Conference on Computer Vision},
	pages={2440--2450},
	year={2025}
}

@article{desideri2012multiple,
	title={Multiple-gradient descent algorithm (MGDA) for multiobjective optimization},
	author={D{\'e}sid{\'e}ri, Jean-Antoine},
	journal={Comptes Rendus. Math{\'e}matique},
	volume={350},
	number={5-6},
	pages={313--318},
	year={2012}
}

\clearpage

\appendix



\addcontentsline{toc}{section}{Appendix} 
\renewcommand \thepart{} 
\renewcommand \partname{}
\part{\Large{~~~~~~~~~~~~~~~~~~~~~~~~~~~~~~~~~~~~~~~~~~~~~~~Appendix}}
\parttoc 

\clearpage


\section{Extended Related Work}
\label{app:related_work}

%

\subsection{Domain Generalization}
\label{app:related_work_dg}

Domain generalization (DG) aims to train a model on data from multiple source 
domains such that it generalizes to previously unseen target domains whose 
distributions may differ substantially from those encountered during 
training~\cite{wang2022generalizing, zhou_domain_2023}. The core challenge is 
that Empirical Risk Minimization (ERM), which minimizes the average loss over 
all training data, may converge to solutions that exploit domain-specific 
shortcuts rather than domain-invariant features, resulting in poor 
out-of-distribution performance. In this section, we will focus on the 
methodological groups that are closest in nature with our proposed method, namely domain-invariant representation learning, sharpness-aware optimization, 
gradient-based alignment, and loss-landscape exploration.

\paragraph{Domain-invariant and robust optimization methods.}
A central line of work found in the literature seeks representations that are 
invariant across training domains. In one of the most fundamental works, Invariant Risk Minimization (IRM)~\cite{arjovsky2019irm} constrains the optimal 
classifier to be identical across all data domains, while 
V-REx~\cite{krueger2021out} penalizes the variance of per-domain risks. 
Fishr~\cite{rame2022fishr} on the other hand enforces consistency of per-domain gradient variances, and CORAL~\cite{sun2016coral} aligns second-order feature statistics across domains. Although these methods directly modify the training 
objective to encourage domain-invariant solutions, they do not directly 
consider the geometry of the loss landscape.

\paragraph{Sharpness-aware minimization for DG.}
As already mentioned in the main text, the observation that flat minima tend to generalize better than sharp ones~\cite{foret2020sharpness, keskar2016large} has motivated the application of SAM and its variants to domain generalization. 
Specifically, GSAM~\cite{zhuang2022surrogate} refines the sharpness estimate via a surrogate gap objective, ASAM~\cite{kwon2021asam} introduces adaptive, scale-aware perturbation magnitudes, GAM~\cite{zhang2023gradient} regularizes the gradient norm to seek first-order flatness, and VaSSO~\cite{li2023enhancing} addresses the noise introduced by mini-batch sampling in the perturbation direction. Furthermore, SAGM~\cite{wang2023sharpness} simultaneously minimizes the empirical loss, the perturbed loss, and the surrogate gap between them, encouraging the clean-loss gradient $\nabla\mathcal{L}(\theta)$ and the perturbed-loss gradient $\nabla\mathcal{L}_p(\theta)$ to remain aligned so that the optimizer converges to regions that are both flat and low-loss. Importantly, the gradient matching is between the clean and perturbed objectives at the \emph{aggregate} level and does not involve per-domain gradient information. Finally, Fisher SAM~\cite{kim2022fisher} re-parameterizes the perturbation using the Fisher information metric to better capture the intrinsic geometry of the model.

A key limitation of vanilla SAM under domain shifts was identified by Zhang et 
al.~\cite{zhang2024domaininspired}, who observed that when domains differ in 
difficulty or data quantity, SAM's gradient-based perturbation is biased toward 
whichever domain has the largest gradient, disrupting training for the others. 
To address this, they proposed DISAM (Domain-Inspired SAM), which incorporates 
a domain-loss variance minimization constraint into the perturbation generation 
step. The resulting perturbation adaptively up-weights well-converged domains 
and down-weights under-converged ones, yielding faster overall convergence and 
improved generalization. DISAM is compatible with other SAM variants (GSAM, 
SAGM) and achieves consistent improvements on the DomainBed benchmark, 
particularly on datasets with large domain gaps such as TerraIncognita. Finally, complementary ensemble-based approaches such as 
SWAD~\cite{cha2021swad} achieve flatness passively by averaging model weights 
densely along the training trajectory, rather than explicitly optimizing a 
sharpness objective.

\paragraph{Gradient-alignment methods for DG.}
A parallel family of approaches leverages the agreement of per-domain gradients 
as a signal for domain invariance. The intuition is straightforward: if a 
model's internal representation does not depend on the domain, then the 
expected gradients of the loss with respect to different domains should point 
in similar directions. Shi et al.~\cite{shi2021gradient} proposed Fish, which 
approximately maximizes the inner product of per-domain gradients. Mansilla et 
al.~\cite{mansilla2021domain} adapted gradient surgery from multi-task 
learning~\cite{yu2020gradient} to DG by muting or randomizing gradient 
dimensions where domains disagree in sign. Le and Woo~\cite{le2024gradient} 
applied gradient alignment to cross-domain face anti-spoofing, demonstrating 
the generality of the approach beyond standard image classification benchmarks.
More recently, Gradient-Guided Annealing (GGA)~\cite{ballas2025gradient} was proposed, for searching parameter configurations via a Simulated Annealing-inspired process, where domain gradients are well-aligned before continuing standard optimization. 

\paragraph{Beyond local flatness: mode connectivity.}
While the methods above focus on properties local to a single minimum 
(flatness, gradient alignment), Li et al.~\cite{liexploring} recently shifted 
attention to a global property of the loss landscape: \emph{mode connectivity}, 
the phenomenon whereby distinct local minima are connected by continuous 
low-loss pathways. They demonstrated empirically that models with poor and 
strong out-of-domain generalization can be connected via such pathways and 
proposed the Billiard Optimization Algorithm (BOA), which navigates these 
pathways by alternating line search (to locate loss contour boundaries) and 
reflection (to redirect the trajectory along the contour). To overcome the 
curse of dimensionality in high-dimensional parameter spaces, BOA operates 
within a low-dimensional Krylov subspace constructed from training gradients 
and Hessian--vector products, which they showed to be well-aligned with test 
gradients across diverse datasets and architectures. BOA achieves 
state-of-the-art results on DomainBed with ViT backbones, outperforming all 
sharpness-aware baselines, and represents a qualitatively different approach to 
loss-landscape exploitation for DG.

\subsection{Multi-Task Learning}
\label{app:related_work_mtl}

Multi-task learning (MTL) trains a single model to perform multiple tasks 
simultaneously, with the goal of exploiting shared structure across tasks to 
improve data efficiency and generalization~\cite{9392366}. The standard 
approach minimizes the average loss across all tasks, however, this simple 
objective often leads to suboptimal solutions because the gradients of 
different tasks may differ substantially in magnitude or direction, a 
phenomenon commonly referred to as \emph{task conflict} or \emph{conflicting 
gradients}~\cite{yu2020gradient, liu2021conflict}. When one task's gradient 
dominates the shared update, other tasks may stagnate or even regress, 
resulting in performance that falls short of what independent single-task 
models can achieve. Addressing this fundamental optimization challenge has 
motivated a rich body of work, which can be broadly organized into 
loss-weighting methods, gradient manipulation methods, and more recent 
approaches that integrate sharpness-aware minimization into the MTL framework.

\paragraph{Loss weighting methods.}
A natural strategy for balancing tasks is to assign scalar weights to each 
task's loss and adapt those weights during training. Kendall et 
al.~\cite{kendall2018multi} proposed Uncertainty Weighting (UW), which uses 
learned task-dependent homoscedastic uncertainty as a proxy for task 
difficulty. Liu et al.~\cite{liu2019end} introduced Dynamic Weight Average 
(DWA), a heuristic that adjusts task weights based on the rate of change of 
each task's loss over recent iterations. Lin et al.~\cite{lin2021reasonable} 
showed that randomly sampling task weights at each step, i.e., Random Loss 
Weighting (RLW), can be surprisingly effective. A scale-invariant (SI) baseline 
that minimizes the sum of log-losses has also been considered, as it is 
invariant to multiplicative rescaling of individual task 
losses~\cite{liu2023famo}. While computationally lightweight, these methods 
adjust only the loss aggregation and do not directly resolve geometric 
conflicts in the gradient space.

\paragraph{Gradient manipulation methods.}
A more principled family of approaches directly modifies the update direction 
to ensure balanced progress across tasks. 
D\'{e}sid\'{e}ri~\cite{desideri2012multiple} introduced the Multiple Gradient 
Descent Algorithm (MGDA), which finds a convex combination of task gradients 
lying in the common descent cone, and Sener and Koltun~\cite{sener2018multi} 
adapted MGDA for deep networks. PCGrad~\cite{yu2020gradient} projects each 
task's gradient onto the normal plane of conflicting gradients from other 
tasks. GradDrop~\cite{chen2020just} randomly zeroes gradient dimensions where 
tasks disagree in sign. CAGrad~\cite{liu2021conflict} constrains the update 
within a ball around the average gradient while maximizing worst-case per-task 
improvement. IMTL-G~\cite{liu2021towards} finds an update direction with equal 
cosine similarity to all task gradients, while ICA~\cite{Senushkin_2023_CVPR} 
decomposes task gradients into independent components. 
Nash-MTL~\cite{navon2022multi} formulates the problem as a bargaining game, and 
FairGrad~\cite{10.5555/3692070.3692179} extends the bargaining framework with a 
fairness-aware conic combination. MoCo~\cite{fernando2023mitigating} corrects 
gradient bias through a momentum-based mechanism. A common limitation of all 
these methods is their $\mathcal{O}(K)$ time and space overhead per iteration, 
which becomes prohibitive when both the number of tasks and the model size are 
large.

\paragraph{Efficient multi-task optimization.}
To address this scalability bottleneck, Liu et al.~\cite{liu2023famo} proposed 
FAMO, which decreases all task losses at approximately equal \emph{rates} by 
working in log-loss space and amortizing the weight update using only the 
change in task losses between consecutive iterations. This achieves 
$\mathcal{O}(1)$ space and time overhead per iteration, matching the complexity 
of standard average-loss training while remaining competitive with gradient 
manipulation methods on standard benchmarks.

\paragraph{Sharpness-aware minimization in MTL.}
Ban et al.~\cite{ban2025samo} provided empirical evidence that 
SAM~\cite{foret2020sharpness} mitigates task conflicts by guiding optimization 
toward flatter loss regions where changes in one task's loss do not 
substantially affect others. Their analysis revealed that applying SAM to the 
average loss (global SAM) and to each task individually (local SAM) both 
contribute to conflict reduction, but neither is consistently superior. Phan et 
al.~\cite{phan2025beyond} proposed F-MTL, which applies SAM independently to 
each task but incurs $K$ additional backpropagations. Ban et 
al.~\cite{ban2025samo} introduced SAMO, which combines global and local 
perturbation information via a weighted average and approximates expensive 
local gradients using a zeroth-order SPSA estimator requiring only forward 
passes. A layerwise normalization scheme stabilizes the estimates, achieving 
the benefits of joint global--local perturbations at a cost comparable to 
global-only SAM.

\subsection{Positioning of SAGE}
\label{app:related_work_positioning}

The methods reviewed above target individual geometric properties of the loss 
landscape, i.e., flatness, gradient alignment, mode connectivity, or task 
balancing, but limited methods address more than one simultaneously. Sharpness-aware methods (SAM, GSAM, SAGM, DISAM) seek flat minima without explicitly considering whether per-distribution gradients agree at those minima. Gradient-alignment methods (Fish, GGA, PCGrad in DG; PCGrad, CAGrad, Nash-MTL in MTL) enforce gradient agreement without considering the curvature of the resulting solution. The closest prior work is that of SAGM~\cite{wang2023sharpness}, which also augments SAM with a gradient-matching 
term. However, the two methods differ in three respects. First, SAGM's 
gradient matching aligns the clean-loss and perturbed-loss gradients, which are 
both computed on the \emph{aggregate} data, to ensure joint descent toward 
flat, low-loss regions and does not compute per-distribution gradients. 

SAGE addresses the curvature through a spectral perturbation that replaces 
SAM's gradient-scaled ascent with the polar factor of each layer's gradient 
matrix, computed via Newton--Schulz iteration and scaled by the layer's 
Frobenius norm. Regarding gradient alignment term, SAGE injects isotropic 
Gaussian noise at the descent step whose magnitude scales with the degree of 
cross-environment gradient conflict, $\beta = \gamma(1 - S(\theta))$, 
discouraging convergence to regions of high gradient disagreement. Unlike 
methods that directly modify the gradient direction (PCGrad, Fish) or solve 
auxiliary optimization problems (Nash-MTL, FAMO), SAGE's noise injection is 
lightweight and requires no per-environment gradient storage beyond what is needed for computing the pairwise cosine similarity. This noise injection was
inspired from the recent GGA~\cite{ballas2025gradient} work.

Importantly, SAGE operates as a drop-in replacement for the optimizer's ascent 
and descent steps, making it compatible with existing gradient manipulation 
methods in MTL (as demonstrated by the SAGE-MGDA, SAGE-FairGrad, and SAGE-SAMO 
combinations in Table~\ref{tab:combined}) and with the standard DomainBed 
protocol in DG (Table~\ref{table:baseline-results}). 

\section{Proofs}
\label{app:proofs}

\subsection{Setup Reminder:}

 For each $e \in \mathcal{E}$, let $\mathcal{L}_e : \Theta \to \mathbb{R}$ denote a per-distribution loss, where $\Theta \subseteq \mathbb{R}^d$ is an open parameter space. The \emph{population risk} is:
\begin{equation}
	R(\theta) := \mathbb{E}_{e \sim \mathcal{P}}[\mathcal{L}_e(\theta)],
\end{equation}
and the \emph{empirical multi-distribution risk} over $K$ iid training distributions $e_1, \dots, e_K \sim \mathcal{P}$ is
\begin{equation}
	\hat{R}(\theta) := \frac{1}{K}\sum_{k=1}^K \mathcal{L}_{e_k}(\theta).
\end{equation}

At any base point $\theta_0 \in \Theta$ we introduce the per-distribution gradient and Hessian,
\vspace{-0.1em}
\begin{equation}
	g_e := \nabla_\theta \mathcal{L}_e(\theta_0), \qquad H_e := \nabla_\theta^2 \mathcal{L}_e(\theta_0),
\end{equation}
and the following cross-distribution statistics,
\vspace{-0.1em}
\begin{equation}
	\bar{g} := \mathbb{E}_e[g_e], \qquad \bar{H} := \mathbb{E}_e[H_e], \qquad \Sigma_g := \mathrm{Cov}_e(g_e) = \mathbb{E}_e\!\left[(g_e - \bar{g})(g_e - \bar{g})^T\right].
\end{equation}

\subsection{Proof of Theorem~\ref{thm:decomposition}: Multi-distribution excess-risk decomposition}
\label{app:thm1_proof}

\begin{proof}[Proof of Theorem~\ref{thm:decomposition}]
	We restate the Theorem and Assumption included in the main text for completeness.
	
	\setcounter{theorem}{0}

\begin{neuraltheorem}
	\begin{theorem}[Multi-distribution excess-risk decomposition]
		\label{thm:decomposition2}
		Under Assumption~\ref{ass:regularity}, let $\hat{\theta}$ denote the minimizer of $\hat{R}$ over $K$ iid training distributions, and let $\xi \sim \mathcal{N}(0, \sigma^2 I)$ denote isotropic Gaussian noise of scale $\sigma > 0$. Then, for a parameter $\theta = \hat\theta + \xi$:
		\begin{equation}
			\mathbb{E}\!\left[\mathbb{E}_{\xi}[R(\theta)]\right] - R(\theta^\star)
			\;=\; 
			\underbrace{\frac{1}{2K}\,\mathrm{tr}\!\left(\bar{H}^{-1}\Sigma_g\right)}_{\text{\normalfont alignment term}}
			\;+\; 
			\underbrace{\frac{\sigma^2}{2}\,\mathrm{tr}(\bar{H})}_{\text{\normalfont curvature term}}
			\;+\; O(K^{-3/2}) \;+\; O(\sigma^3),
			\label{eq:decomposition}
		\end{equation}
		where the outer expectation is over the iid sampling of training distributions, $\bar{H} := \nabla^2 R(\theta^\star)$, and $\Sigma_g := \mathrm{Cov}_e(\nabla\mathcal{L}_e(\theta^\star))$.
	\end{theorem}
\end{neuraltheorem}
	
    \setcounter{assumption}{0}
	\begin{assumption}[Regularity]
		\label{ass:regularity2}
		Each $\mathcal{L}_e$ is three times continuously differentiable in $\theta$, with third derivatives uniformly bounded in a neighborhood of $\theta^\star := \arg\min_\theta R(\theta)$, and $\bar{H}(\theta^\star) \succ 0$ (Positive-Definite). The gradients $g_e(\theta^\star)$ have finite second moments under $\mathcal{P}$.
	\end{assumption}
	
	Assumption~\ref{ass:regularity2} is the standard regularity condition under which quadratic excess-risk expansions are valid, which is however strictly weaker than the convexity assumptions used in classical M-estimation theory \cite{van2000asymptotic}. The $C^3$ requirement is used only to calculate and control the remainder in a second-order Taylor expansion, while the PD-Hessian requirement ensures the quadratic form $\bar{H}^{-1}$ is well defined.
	
	The proof proceeds in three steps: (i) expand $R$ around $\theta^\star$ to obtain the excess risk as a quadratic form in $\hat\theta - \theta^\star$, (ii) characterize $\hat\theta - \theta^\star$ as a sample-mean quantity and take expectations to produce the alignment term, (iii) expand $R$ around $\hat\theta$ under random perturbations $\xi 
	\sim \mathcal{N}(0, \sigma^2 I)$ to produce the curvature term. 
	
	\paragraph{Step 1: Quadratic expansion of $R$ at $\theta^\star$.} Because $\theta^\star$ minimizes the population risk $R$, we have $\nabla R(\theta^\star) = \bar{g} = 0$. Under Assumption~\ref{ass:regularity}, Taylor's theorem gives, for any $\delta \in \mathbb{R}^d$ sufficiently small,
	\begin{equation}
		R(\theta^\star + \delta) = R(\theta^\star) + \nabla R(\theta^\star) \delta + \tfrac{1}{2}\delta^T \nabla^2 R(\theta^\star)\delta + r_3(\delta),
		\label{eq:taylor_R}
	\end{equation}

	where $|r_3(\delta)| \le C\|\delta\|^3$ for some constant $C$ depending on the uniform bound on third derivatives, by definition $\nabla R(\theta^\star) = \bar g = 0$, and $\nabla^2 R(\theta^\star) = \bar H$. Therefore, subtracting $R(\theta^\star)$ yields:
	\begin{equation}
		R(\theta^\star + \delta) - R(\theta^\star) = \tfrac{1}{2}\delta^T \bar{H}\delta + O(\|\delta\|^3).
		\label{eq:excess_quad}
	\end{equation}
	Equation~\eqref{eq:excess_quad} is the standard quadratic excess-risk formula, which holds for any $\delta$. We now look at $\delta = \hat\theta - \theta^\star$ and analyze the distance between the empirical minimizer ($\hat\theta$) and population minimizer ($\theta^\star$).
	
	\paragraph{Step 2: Characterization of $\hat\theta - \theta^\star$.} Note that the empirical minimizer satisfies $\nabla \hat R(\hat\theta) = 0$ across $K$ distributions. Expanding $\nabla \hat R$ around $\theta^\star$ using Taylor's theorem yields,
	
	\begin{equation}
		\nabla \hat R(\hat\theta) = \nabla \hat R(\theta^\star) + \nabla^2 \hat R(\theta^\star) (\hat \theta -\theta^\star) + O(\| \hat \theta - \theta^\star \|^2),
		\label{eq:implicit_0}
	\end{equation}
	or,
	\begin{equation}
		0 = \nabla \hat R(\hat\theta) = \hat g + \hat H(\hat\theta - \theta^\star) + O(\|\hat\theta - \theta^\star\|^2),
		\label{eq:implicit}
	\end{equation}
	where $\hat g := \frac{1}{K}\sum\limits_{k=1}^K g_{e_k}(\theta^\star)$ and $\hat H := \frac{1}{K}\sum\limits_{k=1}^K H_{e_k}(\theta^\star)$. 
	
	Solving for $\hat\theta - \theta^\star$:
	\begin{equation}
		\begin{aligned}
			\hat H(\hat\theta - \theta^\star) &= -\hat g - O(\|\hat\theta - \theta^\star\|^2) \\
			\hat\theta - \theta^\star &= -\hat H^{-1} \hat g - \hat H^{-1} O(\|\hat\theta - \theta^\star\|^2) \\
			\hat\theta - \theta^\star &= -\hat H^{-1}\hat g + O(\|\hat\theta - \theta^\star\|^2) \\
		\end{aligned}
		\label{eq:delta_implicit}
	\end{equation}
	
	Where from Assumption \ref{ass:regularity}, $\bar H$ is PD, so assuming enough data so will $\hat H$. To find the ``size'' of the $(\hat{\theta} - \theta^\star) \approx -\hat{H}^{-1}\hat{g}$, we take the vector norm (magnitude) of both sides. Using the standard property of matrix norms and the triangle inequality, we can separate the matrix from the vector to create an upper bound:
	
	\begin{equation}
		||\hat{\theta} - \theta^\star|| \approx ||-\hat{H}^{-1}\hat{g}|| \le ||\hat{H}^{-1}|| \cdot ||\hat{g}||. 
	\end{equation}
	
	The eigenvalues of an inverse matrix are simply $1 / \lambda$. Because the original eigenvalues ($\lambda$) of $\hat H$ are bounded away from zero, the eigenvalues of the inverse matrix cannot blow up to infinity. As the those eigenvalues are capped, $||\hat{H}^{-1}||$ is just some finite constant number $c$ (based on the largest eigenvalue). Therefore:
	
	\begin{equation}
		\begin{aligned}
			||\hat{\theta} - \theta^\star|| &\le c||\hat{g}|| \\
			||\hat{\theta} - \theta^\star|| &\le O(||\hat{g}||) \\
			||\hat{\theta} - \theta^\star||^2 &\le O(||\hat{g}||^2)
		\end{aligned}
		\label{eq:delta_bigo}
	\end{equation}
	
	And substituting \ref{eq:delta_bigo} to \ref{eq:delta_implicit}, yields:
	
	\begin{equation}
		\hat \theta - \theta^\star = -\hat{H}^{-1}\hat{g} + O(||\hat{g}||^2)
		\label{eq:delta_bigo_clean}
	\end{equation}
	
	This is the standard implicit-function argument used in M-estimator asymptotics \cite{van2000asymptotic}. We now continue to replace $\hat H$ with $\bar H$, which concentrates to $||\hat{H} - \bar{H}||$, as $K \rightarrow \infty$. By the  law of large numbers, as $K$ increases, the 
	empirical average converges to the true expectation: $\hat{H} \rightarrow 
	\bar{H}$. Recall that $\hat{H}$ is the empirical Hessian, calculated 
	by averaging the actual Hessians observed across $K$ randomly sampled 
	training distributions: $\hat{H} = \frac{1}{K}\sum\limits_{k=1}^K H_{e_k}$. 
	Because the training distributions $e_1, \dots, e_K$ are sampled 
	independently and identically (iid) from the meta-distribution 
	$\mathcal{P}$, each observed Hessian $H_{e_k}$ is an independent random 
	matrix. This means $\hat{H}$ is simply a sample mean, and we want to know 
	how far this sample mean deviates from the true mean $\bar{H}$. 

	Let $\hat{H}_{ij}$ denote the scalar entry in the $i$-th row and $j$-th column of the empirical Hessian matrix, and $\bar{H}_{ij}$ denote the corresponding entry in the population Hessian. By definition, the empirical entry is a simple sample mean of the individual distribution Hessians:
	
	$$\hat{H}_{ij} = \frac{1}{K}\sum_{k=1}^K (H_{e_k})_{ij}.$$
	
	Since the third derivatives of the loss are uniformly bounded in a neighborhood of $\theta^*$, as per~\ref{ass:regularity2}, each entry $(H_{e_k})_{ij}$ is itself bounded, and there exists some finite variance 
	$\sigma_{ij}^2 < \infty$ such that:
	
	$$Var((H_{e_k})_{ij}) = \sigma_{ij}^2.$$
	
	Since each distribution is sampled i.i.d, from the Central Limit Theorem  (or Chebyshev's Inequality) for each element; as $K \rightarrow \infty$:
	
	$$\sqrt{K}(\hat{H}_{ij} - \bar{H}_{ij}) \rightarrow \mathcal{N}(0, \sigma_{ij}^2).$$
	
	Where by calculating the Frobenius norm of the difference, it can be easily shown that $||\hat{H} - \bar{H}||_F = O_P(K^{-1/2})$ and we can formally bound the random error matrix:
	
	\begin{equation}
		||\hat{H} - \bar{H}|| = O_P(K^{-1/2})
	\end{equation}
	
	Therefore the distance between the empirical sample $\hat{H}$ and the true population $\overline{H}$ shrinks at a rate proportional to $1/\sqrt{K}$.
	
	Again, following the same logic, by the Central Limit Theorem $\sqrt{K}\hat{g} \rightarrow \mathcal{N}(0, \Sigma_g)$,
	and the magnitude of the average gradient $||\hat{g}||$ shrinks at a rate of $O_P(K^{-1/2})$ and in turn the squared gradient term
	becomes $O(||\hat{g}||^2) = O_P(K^{-1})$.

	Thus, Equation \ref{eq:delta_bigo_clean} becomes:
	\begin{equation}
		\hat\theta - \theta^\star = -\bar H^{-1}\hat g + O_P(K^{-1}).
		\label{eq:delta_clean}
	\end{equation}
	
	\paragraph{Taking expectations.} Substitute Eq.~\eqref{eq:delta_clean} into Eq.~\eqref{eq:excess_quad}:
	\begin{align}
		R(\hat\theta) - R(\theta^\star) &= \tfrac{1}{2}(\bar H^{-1}\hat g)^T \bar H(\bar H^{-1}\hat g) + O_P(K^{-3/2}) \nonumber\\
		&= \tfrac{1}{2}\hat g^T \bar H^{-1}\hat g + O_P(K^{-3/2}),
		\label{eq:excess_hat}
	\end{align}
	where we used $\bar H^{-1}\bar H\bar H^{-1} = \bar H^{-1}$ and the symmetry $(\bar H^{-1})^T = \bar H^{-1}$. The $\|\delta\|^3$ term in Eq.~\eqref{eq:excess_quad} becomes $O_P(K^{-3/2})$ because $\|\hat g\| = O_P(K^{-1/2})$ by the central limit theorem. Now take expectation over the iid sampling of the $K$ training distributions. Using the matrix identity
	\begin{equation}
		\mathbb{E}[x^T A x] = \mathrm{tr}(A\,\mathrm{Cov}(x)) + \mathbb{E}[x]^T A\,\mathbb{E}[x]
		\label{eq:quadratic_identity}
	\end{equation}
	
	(valid for any random vector $x$ with finite second moments and any symmetric $A$) with $x = \hat g$ and $A = \bar H^{-1}$, we compute:
	\begin{itemize}
		\item $\mathbb{E}[\hat g] = \mathbb{E}\!\left[\frac{1}{K}\sum_k g_{e_k}\right] = \bar g = 0$ (since $\theta^\star$ is stationary for $R$), so the mean term vanishes.
		\item $\mathrm{Cov}(\hat g) = \frac{1}{K^2}\sum_k \mathrm{Cov}(g_{e_k}) = \frac{1}{K}\Sigma_g$ by independence of the distributions and due to the fact that each distribution is drawn from $\mathcal{P}$.
	\end{itemize}
	
	Substituting:
	\begin{equation}
		\mathbb{E}[\hat g^T \bar H^{-1}\hat g] = \mathrm{tr}\!\left(\bar H^{-1} \cdot \tfrac{1}{K}\Sigma_g\right) = \tfrac{1}{K}\mathrm{tr}(\bar H^{-1}\Sigma_g).
	\end{equation}
	Combining with Eq.~\eqref{eq:excess_hat}:
	\begin{equation}
		\mathbb{E}[R(\hat\theta) - R(\theta^\star)] = \tfrac{1}{2K}\mathrm{tr}(\bar H^{-1}\Sigma_g) + O(K^{-3/2}).
		\label{eq:alignment_term}
	\end{equation}
	This establishes the alignment term.
	
	\paragraph{Step 3: Expansion of $R$ at $\hat\theta$ under random perturpations.} We want to account for the fact that the learner does not exactly realize $\hat\theta$ and to capture the sharpness of the local landscape, we evaluate the expected risk of the empirical minimizer under random perturbations, e.g. due to optimization noise or finite-precision arithmetic. Let the deployed parameter be $\theta = \hat{\theta} + \xi$, where $\xi \sim \mathcal{N}(0, \sigma^2 I)$ is an isotropic Gaussian perturbation of scale $\sigma > 0$.
	
	By expanding $R$ at $\hat\theta$, from Taylor's theorem:
	
	\begin{equation}
		R(\hat{\theta} + \xi) = R(\hat{\theta}) + \nabla R(\hat{\theta}) \xi  + \frac{1}{2}\xi^T \nabla^2 R(\hat{\theta}) \xi + O(||\xi||^3),
	\end{equation}
	
	and taking expectation with respect to the noise $\xi$,
	
	\begin{equation}
		\mathbb{E}_{\xi}[R(\hat{\theta} + \xi)] = R(\hat{\theta}) + \nabla R(\hat{\theta}) \mathbb{E}[\xi] + \frac{1}{2}\mathbb{E}[\xi^T \nabla^2 R(\hat{\theta}) \xi] + O(\sigma^3)
	\end{equation}

	The linear term vanishes because the perturbation is zero-mean ($\mathbb{E}[\xi] = 0$). For the quadratic term, apply the identity ~\eqref{eq:quadratic_identity} with $x = \xi$ (zero mean) and $A = \nabla^2 R(\hat\theta)$, using $\mathrm{Cov}(\xi) = \sigma^2 I$:
	
	\begin{equation}
		\frac{1}{2}\mathbb{E}[\xi^T \nabla^2 R(\hat{\theta}) \xi] = \frac{\sigma^2}{2}tr(\nabla^2 R(\hat{\theta}))
	\end{equation}
	
	Since $\nabla^2 R(\hat\theta) = \bar H + O_P(K^{-1/2})$ by continuity of the Hessian and the fact that $\hat\theta \to \theta^\star$, we can replace $\nabla^2 R(\hat\theta)$ with the population Hessian $\bar H$ up to a correction that multiplies the existing $\sigma^2/2$ factor and contributes at order $O(\sigma^2 K^{-1/2})$, which is absorbed into the stated error. Thus, the expected risk under perturbation $\xi$ is:
	\begin{equation}
		\mathbb{E}_{\xi}[R(\hat{\theta})] = R(\hat{\theta}) + \frac{\sigma^2}{2}tr(\bar{H}) + O(\sigma^3) + O(\sigma^2 K^{-1/2})
		\label{eq:curvature_term}
	\end{equation}
	
	\paragraph{Combining.} Taking expectation of Eq.~\eqref{eq:curvature_term} over the distribution sampling and adding Eq.~\eqref{eq:alignment_term}:
	\begin{equation}
		\mathbb{E}[\mathbb{E}_{\xi}[R(\hat{\theta})]] - R(\theta^*) = \frac{1}{2K}tr(\bar{H}^{-1}\Sigma_g) + \frac{\sigma^2}{2}tr(\bar{H}) + O(K^{-3/2}) + O(\sigma^3),
	\end{equation}
	which is Eq.~\eqref{eq:decomposition}. \qed
\end{proof}

\subsection{Proof of Counterexample~\ref{thm:decoupling}: Decoupling of flatness and gradient alignment}
\label{app:thm2_proof}

\begin{proof}[Proof of Counterexample~\ref{thm:decoupling}]
	We restate the counterexample included in the main text for completeness.
	
	\setcounter{counterexample}{0}
	\begin{neuraltheorem}
		\begin{counterexample}[Decoupling of flatness and gradient alignment]
			\label{thm:decoupling2}
			Consider the family of multi-distribution learning problems with per-distribution losses
			\begin{equation}
				\mathcal{L}_e(\theta) =
				\tfrac{1}{2}\theta^T A\theta + b_e^T \theta,
				\qquad
				\mathbb{E}_e[b_e] = 0,
				\label{eq:quadratic_family2}
			\end{equation}
			parameterized by a symmetric PD matrix $A \in \mathbb{R}^{d\times d}$ and a collection of linear coefficients $\{b_e\}$ satisfying $\mathbb{E}_e[b_e] = 0$. For every $M > 0$, there exist instances of Eq.~\eqref{eq:quadratic_family2} such that:
			\begin{enumerate}[label=(\roman*)]
				\item \textbf{Flat but misaligned:}
				$\mathrm{tr}(\bar H) \le M^{-1}$
				and
				$\mathrm{tr}(\bar H^{-1}\Sigma_g) \ge M$.
				
				\item \textbf{Aligned but sharp:}
				$\mathrm{tr}(\bar H^{-1}\Sigma_g) \le M^{-1}$
				and
				$\mathrm{tr}(\bar H) \ge M$.
			\end{enumerate}
			In particular, neither $\mathrm{tr}(\bar H)$ nor $\mathrm{tr}(\bar H^{-1}\Sigma_g)$ can be bounded above by any function of the other alone.
		\end{counterexample}
	\end{neuraltheorem}

	For the family in the counterexample we have $\nabla\mathcal{L}_e(\theta) = A\theta + b_e$ and $\nabla^2\mathcal{L}_e(\theta) = A$. Therefore $\bar H = A$ is fixed by the choice of $A$ alone, and the population minimizer is $\theta^\star = -A^{-1}\mathbb{E}_e[b_e] = 0$. At $\theta^\star$, the per-environment gradients are $g_e = b_e$, so
	\begin{equation}
		\Sigma_g = \mathbb{E}_e[b_e b_e^T],
	\end{equation}
	which is fixed by the choice of $\{b_e\}$ alone. The two quantities are thus independently controllable: $A$ determines $\bar H$ without affecting $\Sigma_g$, and $\{b_e\}$ determines $\Sigma_g$ without affecting $\bar H$.
	
	\paragraph{Construction of (i): flat but misaligned.} Take $d = 2$, $A = \lambda I$ with $\lambda = (2M)^{-1}$, and let $b_e$ be supported on $e_2$ with variance $v = 1$ . Then $\bar H = \lambda I$, so
	\begin{equation}
		\mathrm{tr}(\bar H) = 2\lambda = M^{-1},
	\end{equation}
	satisfying the first inequality. Also $\Sigma_g = \mathrm{diag}(0, v)$, so
	\begin{equation}
		\mathrm{tr}(\bar H^{-1}\Sigma_g) = \mathrm{tr}\!\left(\lambda^{-1}I \cdot \mathrm{diag}(0, v)\right) = v/\lambda = 2Mv.
	\end{equation}
	Choosing $v = 1$ gives $\mathrm{tr}(\bar H^{-1}\Sigma_g) = 2M \ge M$. Concretely, two environments with $b_1 = (0, +1)^T$ and $b_2 = (0, -1)^T$ realize this $\Sigma_g$.
	
	\paragraph{Construction of (ii): aligned but sharp.} Take $d = 2$, $A = \lambda I$ with $\lambda = M/2$, and $b_e$ supported on $e_1$ with variance $v = 1/2$. Then
	\begin{equation}
		\mathrm{tr}(\bar H) = 2\lambda = M, \qquad \mathrm{tr}(\bar H^{-1}\Sigma_g) = v/\lambda = M^{-1}.
	\end{equation}
	Two environments with $b_1 = (+1/\sqrt{2}, 0)^T$ and $b_2 = (-1/\sqrt{2}, 0)^T$ realize the required $\Sigma_g$.
	
	\paragraph{Conclusion.} In both constructions, $\bar H$ and $\Sigma_g$ are chosen independently to drive one quantity arbitrarily small and the other arbitrarily large. If there existed a function $\phi$ such that $\mathrm{tr}(\bar H^{-1}\Sigma_g) \le \phi(\mathrm{tr}(\bar H))$ for all instances in the family, then construction~(i) with $M \to \infty$ would contradict $\phi(M^{-1}) < \infty$, symmetrically for the reverse direction. \qed
\end{proof}

\subsection{Limitations}
\label{app:limitations}

\paragraph{Local, not global.} Theorem~\ref{thm:decomposition} is an asymptotic expansion valid in a neighborhood of $\theta^\star$ where the quadratic approximation holds. In non-convex deep learning landscapes, there may be many $\bar g = 0$ points, and the theorem does not address which one the optimizer reaches. This is the ``hypothesis selection'' gap: among the manifold of stationary points, flatness and alignment act as refinement criteria, but the decomposition alone does not prove that optimization converges to the jointly-optimal refinement. 

\paragraph{Not a PAC-Bayes bound.} Although Step 3 of the proof of Theorem~\ref{thm:decomposition} echoes the structure of Gaussian PAC-Bayes analyses the result is an \emph{expectation-form} excess-risk decomposition, not a high-probability uniform bound. A full PAC-Bayes statement would add a KL-divergence term, bound the failure probability explicitly, and require a different argument for the alignment term. The result stated here is sufficient for its purpose, i.e, identifying which quantities appear in the excess risk and in what functional form, but should not be read as a finite-sample generalization guarantee.

\paragraph{Quadratic family is toy.} The decoupling construction in Counterexample~\ref{thm:decoupling} uses a purely quadratic family, in which $\bar H$ is globally constant and independent of $\{b_e\}$. A skeptical reader might object that real deep networks are highly non-convex and that the independence of $\bar H$ and $\Sigma_g$ observed in the quadratic family may not hold globally. The correct response is that Counterexample~\ref{thm:decoupling} provides a strong indication, but no guarantees. If even in a simple setting (exact quadratic, infinite data, closed-form) neither criterion alone suffices, then a fortiori neither is expected to suffice in the harder, non-convex setting.

\section{Motivating Example: Why Both Properties Are Necessary}
\label{app:motivating_example}

Theorem~\ref{thm:decomposition} decomposes the excess risk into an alignment 
term and a curvature term, while Theorem~\ref{thm:decoupling} shows that 
neither term can be bounded as a function of the other. The example below 
illustrates how this is reflected in a concrete setting where each failure mode 
appears along an orthogonal eigendirection of $\bar H$. The example uses a linear model on a Gaussian data-generating process with two domains, chosen so that the loss landscape admits closed-form analysis.

\paragraph{Data-generating process.} Consider a binary classification task with label $y \in \{-1, +1\}$ drawn uniformly, and input $x = (x_\text{inv}, x_\text{spur}) \in \mathbb{R}^2$ consisting of an invariant feature and a spurious feature. For domain $k \in \{1, 2\}$, the features are generated conditionally on $y$ as:
\begin{align}
	x_\text{inv} \mid y &\sim \mathcal{N}(y \cdot \mu_\text{inv},\; \sigma_\text{inv}^2), \label{eq:dgp_inv}\\
	x_\text{spur}^{(k)} \mid y &\sim \mathcal{N}(y \cdot c_k,\; \sigma_\text{spur}^2), \label{eq:dgp_spur}
\end{align}
where $c_1 = +\mu_\text{spur}$ and $c_2 = -\mu_\text{spur}$. The invariant 
feature is causally linked to $y$ identically in both domains, while the 
spurious feature's correlation with $y$ reverses across domains. Crucially, we 
set $\sigma_\text{inv}^2 \gg \sigma_\text{spur}^2$: the invariant feature is 
noisy (high variance) while the spurious feature is clean (low variance). We 
adopt the concrete values $\mu_\text{inv} = 1$, $\sigma_\text{inv}^2 = 9$, 
$\mu_\text{spur} = 2$, $\sigma_\text{spur}^2 = 0.01$ throughout the example.

\paragraph{Loss landscape.} We train a linear model $f(x; \theta) = 
\theta_\text{inv}\, x_\text{inv} + \theta_\text{spur}\, x_\text{spur}$ with 
mean squared error. The expected loss for domain $k$ expands to:

\begin{equation}
	\mathcal{L}_k(\theta) = \frac{1}{2} - \theta_\text{inv}\mu_\text{inv} - \theta_\text{spur} c_k + \frac{1}{2}\theta_\text{inv}^2(\sigma_\text{inv}^2 + \mu_\text{inv}^2) + \frac{1}{2}\theta_\text{spur}^2(\sigma_\text{spur}^2 + \mu_\text{spur}^2) + \theta_\text{inv}\theta_\text{spur}\,\mu_\text{inv} c_k,
	\label{eq:loss_k}
\end{equation}

using the moments $\mathbb{E}[y^2] = 1$, $\mathbb{E}[y\,x_\text{inv}] = 
\mu_\text{inv}$, $\mathbb{E}[x_\text{inv}^2] = \sigma_\text{inv}^2 + 
\mu_\text{inv}^2$, and so on. Because each sub-objective loss is quadratic, the 
Hessian of $\mathcal{L}_k$ is constant:

\begin{equation}
	H_k = \begin{pmatrix} \sigma_\text{inv}^2 + \mu_\text{inv}^2 & \mu_\text{inv}\, c_k \\ \mu_\text{inv}\, c_k & \sigma_\text{spur}^2 + \mu_\text{spur}^2 \end{pmatrix}.
	\label{eq:hessian_k}
\end{equation}

Substituting our numerical values gives $H_{11} = 10$, $H_{22} = 4.01$, and 
$H_{12}^{(1)} = +2$, $H_{12}^{(2)} = -2$. Two features of this landscape are 
critical. First, the curvature along $\theta_\text{inv}$ is $2.5\times$ larger 
than along $\theta_\text{spur}$, meaning the invariant direction forms a narrow 
valley while the spurious direction is comparatively flat. Second, the 
off-diagonal term $\mu_\text{inv} c_k$ flips sign across domains, coupling the 
two parameters differently in each domain's loss surface.

\paragraph{Aggregate loss and its minima.} The aggregate (ERM) loss $\bar{\mathcal{L}} = \frac{1}{2}(\mathcal{L}_1 + \mathcal{L}_2)$ has the Hessian:
\begin{equation}
	\bar{H} = \frac{1}{2}(H_1 + H_2) = \begin{pmatrix} 10 & 0 \\ 0 & 4.01 \end{pmatrix},
	\label{eq:hessian_agg}
\end{equation}
where the cross-domain averaging cancels the off-diagonal terms. The aggregate 
minimum is $\theta^* = \bar{H}^{-1}\nabla_\theta\bar{\mathcal{L}}(0) = (0.1, 
0)^T$, confirming that the ERM optimum correctly assigns zero weight to the 
spurious feature.

However, the individual domain minima tell a different story. Setting 
$\nabla_\theta \mathcal{L}_k = 0$ and solving yields domain-specific optima 
that place substantial weight on $\theta_\text{spur}$ with opposite signs 
across domains. In each domain individually, the spurious feature is a 
low-noise, high-signal predictor and can lead to training models that will not
generalize.

\paragraph{Failure mode 1: Flatness without alignment.} A method that searches 
for low-curvature solutions within a single domain finds the per-domain optima 
$\theta_k^* = H_k^{-1}b_k \approx (0.0003, \pm 0.499)^T$, which lie almost 
entirely along the spurious axis with reversed signs across domains. This is a 
structural consequence of the within-domain feature statistics, as the spurious 
feature has signal-to-noise ratio $\mu_\text{spur}^2/\sigma_\text{spur}^2 = 
400$, far exceeding the invariant feature's $\mu_\text{inv}^2/\sigma_\text{inv}^2 \approx 0.11$. Within either domain 
individually, the spurious feature is the better predictor. A flatness-only criterion cannot distinguish a flat minimum that uses spurious features from 
one that uses invariant ones, and the alignment term of Theorem~\ref{thm:decomposition} is what penalizes the difference.

\paragraph{Failure mode 2: Alignment without flatness.}  A method that enforces gradient agreement across domains correctly identifies the shared direction. The per-domain gradients at the origin are $\nabla\mathcal{L}_1(\mathbf{0}) = (-1, -2)^T$ and $\nabla\mathcal{L}_2(\mathbf{0}) = (-1, +2)^T$, which agree on 
the invariant component and conflict on the spurious one. The optimizer converges to $\theta^* = (0.1, 0)^T$, i.e. the aggregate minimum, which uses exclusively the invariant feature. However, $\theta^*$ sits on the 
\textit{sharp} eigendirection of $\bar H$ (eigenvalue 10), so a parameter 
displacement $\delta$ along $\theta_\text{inv}$ leads to a loss increase of $\frac{1}{2} \cdot 10 \cdot \delta^2 = 5\delta^2$. Therefore, under distribution shift the optimal coefficient moves and the steep curvature amplifies the resulting error. The curvature term of 
Theorem~\ref{thm:decomposition} is what penalizes this.

\paragraph{Connection to Counterexample~\ref{thm:decoupling}.} The two failure modes 
occur along orthogonal eigendirections of $\bar H$, which is what makes this two-dimensional construction minimal. The cross-domain gradient covariance at the aggregate minimum is $\Sigma_g = \mathrm{diag}(0, 3.24)$, supported entirely on the flat eigendirection of $\bar H$. The alignment term evaluates 
to $\mathrm{tr}(\bar H^{-1}\Sigma_g) = 3.24/4.01 \approx 0.81$, which is 
dominated by the ratio of gradient variance to the small eigenvalue. This is, 
structurally, an instance of construction~(i) of Counterexample~\ref{thm:decoupling}, as the support of $\Sigma_g$ coincides with the flat eigendirection of $\bar 
H$, so the alignment term is amplified by the inverse curvature in that 
direction. The two failure modes the example exhibits are not just numerical 
coincidences but the structural examples of the connection between $\Sigma_g$ 
and $\bar H$.


\section{Scale Invariance and Scale-Dependent Flatness in SAM}
\label{app:scale_invariance}

A well-documented limitation that SAGE addresses is with regards to the interaction between SAM's fixed perturbation radius and the scale invariance of 
modern architectures~\cite{dinh2017sharp}. Neural networks employing ReLU 
activations and normalization layers exhibit positive scale invariance, meaning that scaling a weight matrix $W$ by $\alpha > 0$ to obtain $W' = \alpha W$ (with the subsequent layer scaled by $1/\alpha$) leaves the network's output unchanged. The loss landscape is therefore functionally identical, yet the gradient scales inversely:

\begin{equation}\label{eq:scale_inv}
	\nabla_{W'} \mathcal{L} = \frac{1}{\alpha} \nabla_W \mathcal{L}.
\end{equation}

Because SAM uses a fixed $L_2$-radius $\rho$, the perturbation's size relative 
to the parameter norm changes drastically with $\alpha$. Shrinking the weights 
($\alpha < 1$) makes the fixed-size perturbation push parameters proportionally 
further from their current values, inflating the apparent sharpness; 
conversely, scaling weights up ($\alpha \gg 1$) makes the same perturbation 
negligibly small relative to the parameter norm, producing artificially flat 
sharpness estimates. This phenomenon, termed ``scale-dependent 
flatness''~\cite{kwon2021asam}, means that SAM's sharpness measure conflates 
the true geometry of the loss surface with the arbitrary scale of the 
parameterization.

\begin{figure}[htbp]
	\centering
	\includegraphics[width=0.65\textwidth]{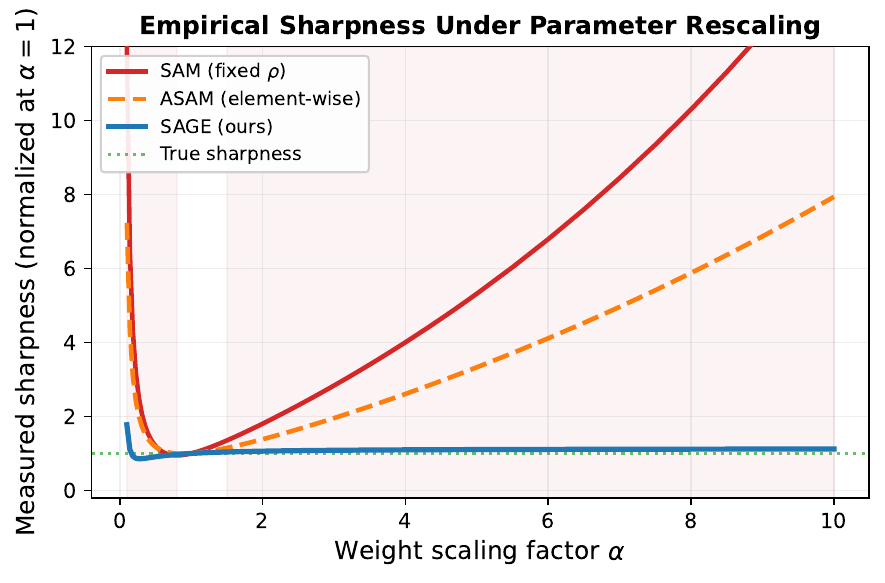}
	\caption{\textbf{Empirical scale invariance.} Measured sharpness as a function of weight rescaling factor $\alpha$ for a two-layer MLP (no 
		normalization layers). The network computes the same function at all 
		$\alpha$, so the true sharpness (green, dotted) is constant. Both SAM (red) and ASAM (orange, dashed) produce scale-dependent sharpness estimates that diverge away from $\alpha\!=\!1$, despite the underlying function being identical. SAGE (blue) is approximately invariant to the rescaling, correctly reflecting the true geometry.}
	\label{fig:fake_flatness}
\end{figure}

\subsection{Illustrative Example}

To empirically demonstrate the scale-invariance failure, we train a two-layer 
MLP without normalization layers ($\texttt{Linear}(2, 64) \to \texttt{ReLU} \to 
\texttt{Linear}(64, 2)$) on a synthetic concentric-circles task until 
convergence. At the converged parameters $\theta^*$, we apply the rescaling 
$W_1 \mapsto \alpha W_1$, $b_1 \mapsto \alpha b_1$, $W_2 \mapsto W_2 / \alpha$ 
for $\alpha \in [0.1, 10.0]$. As previously mentioned, the network 
computes the same function for all $\alpha$. For this example, we measure the sharpness $\mathcal{L}(\theta + \epsilon) - \mathcal{L}(\theta)$ at each 
$\alpha$ using SAM (fixed $\rho$), ASAM (element-wise adaptive), and SAGE 
(spectral, scale-adaptive). As shown in Figure~\ref{fig:fake_flatness}, both 
SAM and ASAM produce sharpness estimates that diverge 
away from $\alpha = 1$, while SAGE remains approximately constant, correctly 
reflecting the invariance of the underlying loss geometry\footnote{The effect on SAM's fixed-$\rho$ perturbation depends on which layer dominates. In this architecture, $W_2$'s gradient grows and $\alpha$ causes the perturbation to push harder on $W_2$, thus increasing sharpness for large $\alpha$.}.

\section{Experimental details \& Hyperparameters}
\label{app:experiments}

\subsection{Details \& Hyperparameters}
\paragraph{Domain Generalization}
For DG, we follow the protocol of the widely adopted and challenging DomainBed~\cite{gulrajani2020domainbed} benchmark, and conduct experiments on 
five image classification datasets. Specifically, we follow the 
leave-one-domain-out evaluation protocol for PACS~\cite{li2017deeper} (9{,}991 
images, 4 domains and 7 classes), VLCS~\cite{fang2013unbiased} (10{,}729 
images, 4 domains and 5 classes), OfficeHome~\cite{venkateswara2017deep} 
(15{,}588 images, 4 domains and 65 classes), TerraIncognita~\cite{beery2018recognition} (24{,}788 images, 4
domains and 10 classes), and DomainNet~\cite{peng2019moment} (586{,}575 images, 
6 domains and 345 classes), and report the average top-1 accuracy over 3 runs 
based on training-domain split validation. Following previous implementations~\cite{liexploring, zhang2024domaininspired}, we select a batch size of 12 where images are uniformly sampled across source domains, set an initial learning rate of $1e-6$ and employ the Adam optimizer for finetuning a 
CLIP~\cite{radford2021learning} pretrained ViT-B/16~\cite{dosovitskiy2021an} . 
Regarding SAGE-specific hyperparameters, we select $\rho=1e-6$ 
and $\gamma = 1e-6$ across all datasets.

\paragraph{Multi-Task Learning}
For the MTL setting, we follow the standard protocol in recent literature~\cite{ban2025samo, liu2023famo, navon2022multi} and evaluate on Cityscapes~\cite{cordts2016cityscapes}, a dataset of urban street scenes containing $5{,}000$ images with pixel-level annotations, and it supports two tasks: 7-class semantic segmentation and depth estimation. In contrast, NYU-v2~\cite{silberman2012indoor} focuses on indoor environments, providing $1{,}449$ densely annotated images and enabling three tasks: 13-class semantic segmentation, depth estimation, and surface normal prediction. For training, we employ MTAN~\cite{liu2019end} as the shared backbone, with task-specific attention modules built on top of a SegNet~\cite{badrinarayanan2017segnet}. 
Following previous implementations~\cite{ban2025samo, liu2023famo} the model is 
trained for 200 epochs with a batch size of 8 for Cityscapes and 2 for NYU-v2. 
The learning rate is set to $1e-4$ for the first 100 epochs and is
then halved for the remainder. Regarding SAGE hyperparameters, we set $\rho = 1e-4$ and $\gamma = 1e-5$ for both datasets.

\subsection{Infrastructure}
All experiments were conducted on a cluster containing $4 \times 40$ GB NVIDIA A100 GPU cards, split into 8 20GB virtual MIG devices and $1 \times$ 24GB NVIDIA RTX A5000 GPU card, via a SLURM workload manager.

\section{Additional Ablations}
\label{app:ablations}
In Figure \ref{fig:grad_agreement} we provide plots of the cosine similarity between the average training-domain gradient and the gradient on the unseen target during training for the DomainBed datasets. The similarity is calculated at every training step , while the lines are smoothed for better visibility.

\begin{figure}[htbp]
	\centering
	
	\begin{subfigure}[t]{0.9\linewidth}
		\centering
		\includegraphics[width=\linewidth]{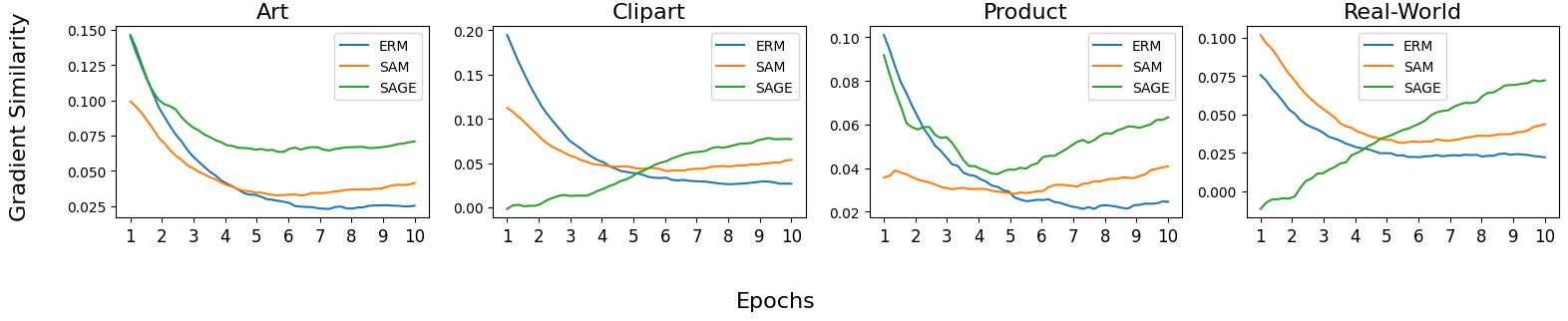}
		\caption{OfficeHome}
		\label{fig:grad_agreement_oh}
	\end{subfigure}
	
	\vspace{0.5em}
	
	\begin{subfigure}[t]{0.9\linewidth}
		\centering
		\includegraphics[width=\linewidth]{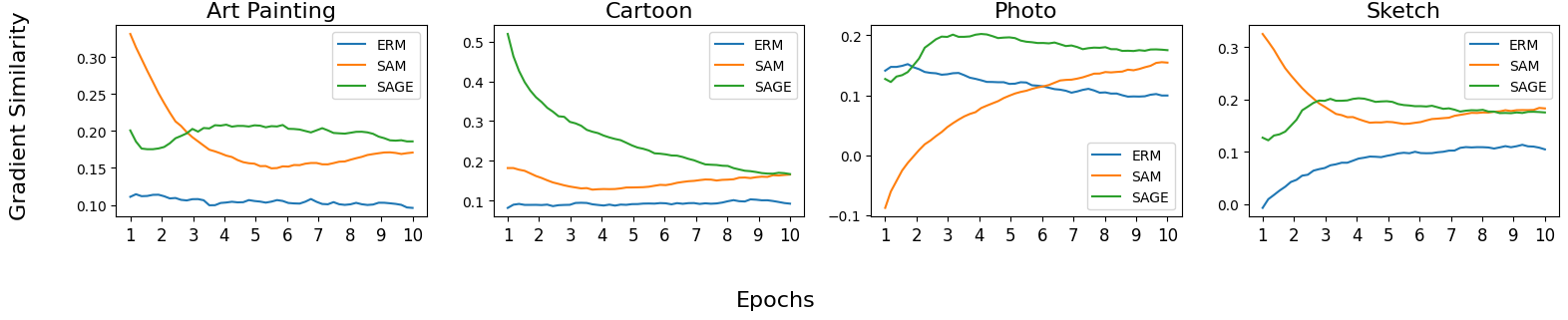}
		\caption{PACS}
		\label{fig:grad_agreement_pacs}
	\end{subfigure}
	
	\vspace{0.5em}
	
	\begin{subfigure}[t]{0.9\linewidth}
		\centering
		\includegraphics[width=\linewidth]{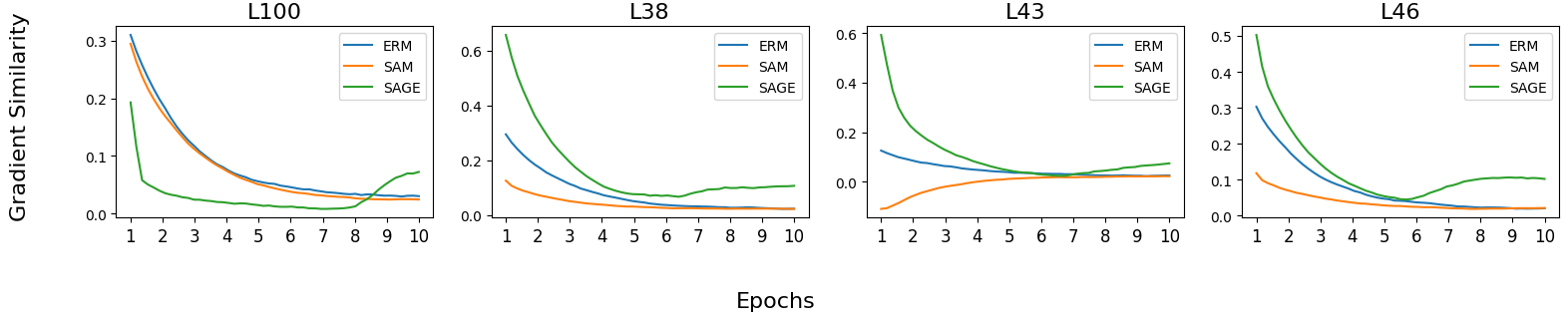}
		\caption{TerraIncognita}
		\label{fig:grad_agreement_ti}
	\end{subfigure}
	
	\vspace{0.5em}
	
	\begin{subfigure}[t]{0.9\linewidth}
		\centering
		\includegraphics[width=\linewidth]{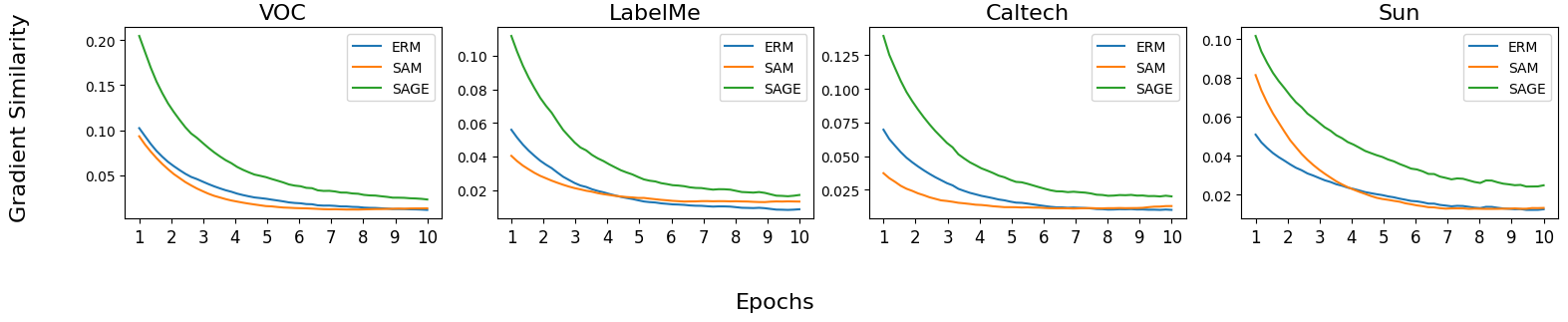}
		\caption{VLCS}
		\label{fig:grad_agreement_vlcs}
	\end{subfigure}
	
	\caption{Gradient alignment for unseen domains during model training on the datasets of DomainBed.}
	\label{fig:grad_agreement}
\end{figure}

\section{Computational Analysis}
\label{app:complexity}

Each SAGE training iteration requires two forward-backward passes through the 
network, identical in structure to SAM and all SAM-like methods. The first pass 
computes the aggregate and per-environment gradients $g_{e_k}$
from a single forward call over the concatenated mini-batches from all $K$ distributions (domains in DG or tasks in MTL), incurring no additional forward passes beyond what ERM already performs with the same batch. From these 
per-distribution gradients, the pairwise cosine similarity $S(\theta)$
and the noise scale $\beta$ are computed in $\mathcal{O}(K^2 d)$
time, which is negligible relative to the cost of backpropagation for any practical number of environments ($K \leq 6$ in all of our experiments). The spectral perturbation replaces SAM's $L_2$-normalization of the gradient with 
$T = 5$ Newton-Schulz iterations per weight matrix, each 
consisting of a single matrix multiplication and a matrix subtraction. These 
operations are fully parallelized on modern GPU hardware and add negligible 
overhead. The second forward-backward pass at the perturbed point $\theta + \epsilon$ is identical to that of SAM. Finally, the noise injection
$\beta\,\xi$ at the descent step amounts to sampling a Gaussian vector and a single multiply-add operation, which is $\mathcal{O}(d)$. In total, SAGE's per-iteration cost is approximately $2\times2$ that of ERM, the same factor as SAM, with the additional operations (cosine similarities, Newton-Schulz iterations, and noise sampling) contributing an
overhead that is negligible in practice compared to the forward-backward 
passes. Importantly, inference is entirely unaffected by SAGE, as all 
modifications occur exclusively during the training optimization loop.

\section{Pseudo code of the proposed SAGE algorithm}
\label{app:sage_algorithm}

\begin{algorithm}[h]
	\caption{SAGE: Spectral-Aware and Gradient-Aligned Exploration}
	\label{alg:sage}
	\small
	\begin{algorithmic}[1]
		\REQUIRE Learning rate $\eta$, perturbation radius $\rho$, noise scale $\gamma$, sub-objectives $\{D_1, \dots, D_K\}$, feature extractor $f_\theta$, cosine similarity function ($cos$), Newton-Schulz iterations $T$, loss function $\ell$, base optimizer OPT(), total number of training operations $n$.
		\FOR{t $\leftarrow$ 1 to $n$}
		\STATE Sample mini-batches $\{(x_k, y_k)\}_{k=1}^K$ from each sub-objective
		\STATE \textbf{// Phase 1: Sub-objective and total gradients}
		\STATE Compute logits $\hat{y} = f_\theta(\text{concat}(x_1, \dots, x_K))$
		\FOR{$k = 1, \dots, K$}
		\STATE $\mathcal{L}_k \leftarrow \ell(\hat{y}_k, y_k)$ 
		\STATE $g_k \leftarrow \nabla_\theta \mathcal{L}_k$
		\ENDFOR
		\STATE $\bar{g} \leftarrow \sum_{k} w_k \, g_k$  \hfill \COMMENT{sample-weighted total gradient}
		\STATE \textbf{// Phase 2: Calculate gradient agreement \& noise scale}
		\STATE $S \leftarrow \frac{2}{K(K-1)} \sum_{i<j} \cos(g_i, g_j)$
		\STATE $\beta \leftarrow \gamma \, (1 - S)$
		\STATE \textbf{// Phase 3: Compute Spectral Perturbations (Muon Logic)}
		\FOR{each weight matrix $W \in \theta$ and its corresponding gradient $G \in g$}
		\IF{$G$ is a matrix or higher-order tensor}
		\STATE $G_{\text{ortho}} \leftarrow \text{NewtonSchulz5}(G)$ \hfill \COMMENT{Orthogonalize via Newton-Schulz}
		\STATE $\epsilon_W \leftarrow \rho \|W\|_F G_{\text{ortho}}$ \hfill \COMMENT{Scaled spectral perturbation}
		\ELSE
		\STATE $\epsilon_W \leftarrow \rho \frac{G}{\|G\|_2}$
		\ENDIF
		\ENDFOR
		\STATE \textbf{// Phase 4: Apply Perturbation (Ascent Step)}
		\STATE $\theta \leftarrow \theta + \epsilon$
		
		\STATE \textbf{// Phase 5: Adversarial Forward \& Backward Pass}
		\STATE Compute perturbed loss $\mathcal{L}_{\text{pert}} \leftarrow \ell(f_\theta(\text{concat}(x_1, \dots, x_K)))$
		\STATE Compute perturbed gradient $g_{\text{pert}} \leftarrow \nabla_{\theta} \mathcal{L}_{\text{pert}}$
		
		\STATE \textbf{// Phase 6: Restore Weights \& Inject Gradient Noise (Descent Step)}
		\STATE $\theta \leftarrow \theta - \epsilon$ \hfill
		\COMMENT{Restore clean weights}
		\STATE $g_{\text{final}} \leftarrow g_{\text{pert}} + \beta \mathcal{N}(0, \mathbf{I})$ \hfill \COMMENT{Inject domain-conflict noise}
		
		\STATE \textbf{// Phase 7: Optimizer Step}
		\STATE $\theta \leftarrow \text{OPT}(\theta,\eta, g_{\text{final}})$
		\ENDFOR
	\end{algorithmic}
\end{algorithm}



\clearpage

\end{document}